\documentclass{article}
\usepackage{iclr2027_conference,times}

\usepackage{amsmath,amsfonts,bm}

\def\eqref#1{equation~\ref{#1}}

\def\1{\bm{1}}

\DeclareMathAlphabet{\mathsfit}{\encodingdefault}{\sfdefault}{m}{sl}
\SetMathAlphabet{\mathsfit}{bold}{\encodingdefault}{\sfdefault}{bx}{n}

\usepackage{hyperref}
\hypersetup{colorlinks=true,linkcolor=[rgb]{0.12,0.25,0.55},citecolor=[rgb]{0.12,0.25,0.55},urlcolor=[rgb]{0.12,0.25,0.55}}
\usepackage{url}
\usepackage{graphicx}
\usepackage{booktabs}
\usepackage{amsmath}
\usepackage{amssymb}
\usepackage{multirow}
\usepackage{xcolor}
\definecolor{offsetgreen}{RGB}{52,135,72}
\definecolor{residorange}{RGB}{200,100,45}
\newcommand{\offb}{{\color{offsetgreen}\bar b}}
\newcommand{\resr}{{\color{residorange}r_t}}
\usepackage{amsthm}
\newtheorem{proposition}{Proposition}

\title{\raggedright\hyphenpenalty=10000 \exhyphenpenalty=10000 ECHO: Emergent Common-Mode Hidden Offsets Drive Vocabulary Collapse in Preference Optimization}

\author{Yun Li$^{1}$\thanks{Corresponding author.} \quad Ehsan Javanmardi$^{1}$ \quad Simon Thompson$^{2}$ \quad Jiachen Gong$^{1}$ \\
\bfseries Qunli Zhang$^{3}$ \quad Zifan Zeng$^{3}$ \quad Shiming Liu$^{3}$ \quad Peng Wang$^{3}$ \\
\bfseries Zixuan Guo$^{1}$ \quad Manabu Tsukada$^{1}$ \\[4pt]
\normalfont\small $^{1}$Graduate School of Information Science and Technology, The University of Tokyo, Tokyo, Japan \\
\normalfont\small $^{2}$TIER IV, Inc., Tokyo, Japan \qquad $^{3}$Huawei \\[2pt]
\normalfont\small \texttt{\{li-yun,ejavanmardi,jiachengong,zixuan,mtsukada\}@g.ecc.u-tokyo.ac.jp} \\
\normalfont\small \texttt{simon.thompson@tier4.jp} \\
\normalfont\small \texttt{\{zhangqunli1,zifan.zeng,liushiming3,wangpeng394\}@huawei.com}}

\iclrfinalcopy

\begin{document}

\maketitle
\lhead{Preprint}

\begin{abstract}

Preference optimization can improve every ranking metric while building a hidden, shared representation shift that the relative loss barely sees, one that can eventually destroy generation. We trace this failure to the \emph{common-mode offset}, a shared directional shift in the model's hidden states that appears across token positions, prompts, and both chosen and rejected responses. Because it moves chosen and rejected representations together, the relative ranking loss barely reacts to its accumulation. As this offset grows, it projects onto the output layer to form a \emph{vocabulary well}, elevating a single token's probability until it dominates decoding and causes generation collapse. We demonstrate the causal role of this mechanism through targeted interventions. The isolated offset alone reconstructs the vocabulary well, injecting the learned offset into an untreated base model recreates the collapse where norm-matched random directions do not, and masking this common component from the loss prevents the well in our flagship setting. Under the per-token ranking loss we study, the same collapse appears in 14 of 17 base-model lineages and under full-parameter fine-tuning. Moreover, the offset is not merely an artifact of collapse, since it is already present in textbook DPO and SimPO runs that form no well. Textbook DPO keeps it weakly aligned with any single vocabulary direction even with three times our base budget. Trained from base models for that budget, however, full-parameter SimPO drives it into collapse on both Llama-2 and Llama-3.1, and the same kind of offset grows, without a well, from a public SFT checkpoint under the official SimPO recipe. Across 17 UltraFeedback runs, the better a run ranks held-out responses, the less often two public reward models prefer its samples to the base model's. These findings expose a blind spot in standard ranking metrics and introduce a representation-level signal that can warn of generation collapse before it occurs.

\end{abstract}

\section{Introduction}

Preference optimization has become the standard paradigm for adapting pretrained language models toward human-preferred responses \citep{ouyang2022training}, ranging from pairwise Direct Preference Optimization (DPO) \citep{rafailov2023dpo} to listwise objectives based on Plackett--Luce ranking \citep{plackett1975analysis,luce1959individual,liu2025lipo}. Recent pipelines push this ranking signal further, with repeated rounds of preference optimization on continually refreshed preference data \citep{Yuan2024-mp,Grattafiori2024-rc}. However, recent work reveals a growing disconnect between successful ranking optimization and healthy text generation \citep{Rafailov2024-bz}. \citet{razin2025unintentional} identify \emph{likelihood displacement}, where the absolute likelihood of the chosen response can plummet even as its relative preference margin improves. \citet{ren2025learning} show that suppressing rejected tokens can undesirably squeeze probability mass toward tokens the base model already favored. \citet{yuan2025common} mathematically explain why chosen and rejected likelihoods move together via \emph{gradient entanglement}. While these studies illuminate how probabilities shift during ranking, the terminal state of this dynamic remains largely unexplained. How does this movement manifest inside the model, why does it eventually fixate on a seemingly irrelevant token, and why do standard ranking metrics look healthy while generation breaks down?

\begin{figure}[t]
\centering
\includegraphics[width=0.92\textwidth]{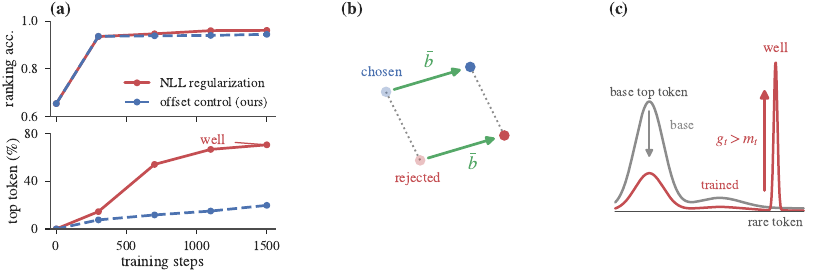}
\caption{(a) Two Llama-2 runs on the same driving lists rank held-out responses equally well, but only the one without offset control collapses onto one token. (b) The common-mode offset $\offb$ moves chosen and rejected hidden states together, so the margin the loss sees barely changes. (c) The offset lifts a rare token above the base model's top token once its gain $g_t$ exceeds the base gap $m_t$.}
\label{fig:teaser}
\end{figure}

We find that preference training induces a surprisingly simple internal anomaly. The final hidden states begin to drift by a shared vector across token positions, prompts, and both chosen and rejected responses. We term this shared movement the \emph{common-mode offset}. Because this offset moves the scores of chosen and rejected responses by nearly the same amount, a ranking loss that optimizes their relative difference has almost no incentive to penalize it. Consequently, this shared shift can keep growing in the background, masked by improving preference margins.

This hidden shift becomes destructive because the output layer projects these hidden states into vocabulary logits. When the offset pushes a single token's logit far higher than the tokens the base model would predict, it creates a \emph{vocabulary well}. Once this well becomes deep enough to dictate decoding, the model enters \emph{generation collapse}, uncontrollably repeating that single token. Importantly, the well need not be a rare token with an unusually large output-vector norm, since the same mechanism produces a well on the ordinary token \texttt{Conf} on UltraFeedback \citep{cui2024ultrafeedback}.

Causal interventions confirm that this offset causes the well. Keeping only its common component reconstructs the well, injecting it into an untreated base model recreates the collapse where norm-matched random directions do not, and centering the rejected-side loss to hide it prevents the well in the flagship setting (\S\ref{sec:mechanism}). The offset moves both chosen and rejected scores by tens of nats but their margin by less than 3 nat in an average list, which is why held-out ranking improves while generation collapses (Figure~\ref{fig:teaser}). Across 17 UltraFeedback runs, the better a run ranks held-out responses, the less often two public reward models prefer its samples to the base model's (\S\ref{sec:judge}).

This vulnerability is not an artifact of one unusual hyperparameter setting. Our main objective is the ranking term of PRO \citep{Song2024-kp}, and we observe the same collapse signature in 14 of 17 base-model lineages and under full-parameter fine-tuning (\S\ref{sec:generalization}). Moreover, the common-mode offset grows under every objective we test, including textbook DPO and SimPO, and the objectives differ in how hard they push it. We measure this push by the loss's \emph{demand}, the per-token margin beyond which it stops pushing. Textbook DPO at its usual setting has a low demand, so its offset stays weakly aligned with any vocabulary direction and generation stays intact. This protection does not come from the reference model, since adding a reference model to the per-token objective does not prevent the well (\S\ref{sec:demand}). SimPO lies between the two. Trained from base models for three times our base budget, full-parameter SimPO drives the offset into collapse, and the offset also grows, without forming a well, from a public SFT checkpoint under the official SimPO recipe.

Therefore, generation collapse is not an inevitability of preference optimization, but the extreme terminal state of a hidden dynamic that is active long before generation visibly fails. Hiding the offset from the loss and penalizing it prevents collapse up to three times the base budget in the four base models we tested.

Prior work showed that preference optimization can move probability mass where it should not go and couples chosen and rejected updates. We link these probabilistic symptoms to one translation vector in the final hidden state that the loss barely sees, show that it causes the well, and give a necessary condition for its onset. This exposes a blind spot in ranking metrics.

\section{Related Work}
\label{sec:related}

\textbf{Likelihood displacement, squeezing, and gradient entanglement.} \citet{razin2025unintentional} demonstrate that DPO training can decrease the chosen response's absolute likelihood and predict \emph{when} a dataset risks it from the similarity of the responses' hidden embeddings. \citet{yuan2025common} prove that margin-based losses entangle the gradients of chosen and rejected log-probabilities, explaining \emph{why} the two tend to move in tandem. While these works map the symptoms in probability space, we isolate in representation space the \emph{common-mode offset} of the final hidden state, which produces the well. Furthermore, \citet{ren2025learning} show that the negative gradient on rejected tokens ``squeezes'' probability mass toward tokens the base model is already confident about. The vocabulary wells we discover are different, since the base model initially ranks our well tokens far below its preferred tokens, typically by 17 to 39 nat, until the offset inflates their logits to dominate decoding.

\textbf{Regularizing absolute likelihood.} DPO-Positive \citep{pal2024smaug}, an explicit NLL term \citep{pang2024iterative}, and N-DPO, a length-normalized regularization toward the reference \citep{im2026normalized}, constrain scalar probabilities rather than the shared direction in the hidden state. In our experiments, an NLL term on the chosen response does not prevent the well (\S\ref{sec:recipe}).

\textbf{Unembedding geometry and representation anisotropy.} Pretrained embeddings and hidden states share a common mean, a few dominant directions, and outlier dimensions \citep{Mu2018-xl,Ethayarajh2019-gv,Bis2021-tl,Gao2019-sp,Puccetti2022-ri,macocco2025outlier,sun2024massive,stolfo2024confidence}, whereas the common-mode offset is built during preference \emph{training}. Functionally, it resembles activation steering \citep{Rimsky2024-to}, and DPO is known to change behavior by shifting the residual stream \citep{Lee2024-df}, but here preference optimization grows the steering vector on its own until generation collapses. Many of our severe wells emerge from under-trained tokens \citep{land2024fishing}, which merely provide a path of least resistance for the offset, and decoding-time corrections \citep{ding2026breaking} leave the training-time cause in place. Appendix~\ref{app:related} gives more detail.

\section{Setup}
\label{sec:setup}

\paragraph{Listwise and pairwise preference objectives.} Our main runs train the model to rank a list of responses. Given a prompt $x$ and $K$ responses ranked from most to least preferred, $y^{(1)} \succ y^{(2)} \succ \dots \succ y^{(K)}$, the sequential Plackett--Luce objective \citep{plackett1975analysis,luce1959individual} asks the model to prefer the top remaining response over all responses below it, one stage at a time, by minimizing
\[\mathcal{L}_{\text{PL}} = -\sum_{k=1}^{K-1} \log \frac{\exp(\beta \, s_\theta(x, y^{(k)}))}{\sum_{j=k}^{K} \exp(\beta \, s_\theta(x, y^{(j)}))}, \qquad s_\theta(x,y) = \frac{1}{\vert{}y\vert{}}\sum_{t=1}^{\vert{}y\vert{}} \log \pi_\theta(y_t \mid x, y_{<t}).\]
Each response is scored by the average per-token log-probability of the policy $\pi_\theta$, without a reference model, and $\beta>0$ sets the scale of the score differences. Up to this scale, the objective recovers the ranking term of PRO \citep{Song2024-kp}, which also adds a supervised term on the top response. For $K=2$ it becomes the length-normalized, reference-free pairwise objective of SimPO \citep{meng2024simpo} without its target margin $\gamma$, which we call \emph{pairwise ranking}.

We train on two preference datasets. The first comprises 5,383 driving-preference lists adapted from \citet{Li2025-ka}, featuring four responses ranked by the severity of their physical consequences. The second is UltraFeedback \citep{cui2024ultrafeedback}, mostly used as pairs. As a boundary case we also train textbook DPO \citep{rafailov2023dpo}, which scores a whole response by its summed log-probability ratio against a frozen base model, $s_\theta(x,y) = \log \pi_\theta(y \mid x) - \log \pi_{\text{ref}}(y \mid x)$. Of the many preference objectives, including IPO \citep{Gheshlaghi-Azar2024-en}, ORPO \citep{Hong2024-gs}, and KTO \citep{Ethayarajh2024-zd}, we study per-token ranking, SimPO, and textbook DPO, whose demands span two orders of magnitude (\S\ref{sec:demand}). We primarily train LoRA adapters \citep{Hu2021-me}, optionally adding an NLL term $\lambda_{\text{NLL}} \cdot \text{NLL}(y^{(1)})$ on the chosen response \citep{pang2024iterative}, and some runs on Llama-2, Llama-3.1, and the Llama-3-8B SFT checkpoint use full-parameter fine-tuning.

\paragraph{A nearly unconstrained translation.} Every ranking objective above depends only on differences between scores. Adding the same amount to the scores of all responses of a prompt therefore leaves the loss unchanged. This invariance makes rewards under the Bradley--Terry \citep{Bradley1952-zi} and Plackett--Luce models identifiable only up to a prompt-dependent constant \citep{rafailov2023dpo}. A shift of the hidden states shared by every position does not add exactly the same amount to every score, since its effect on each token depends on the token and its context. To first order, however, the objective sees such a shift only through its component along one direction, a weighted difference of the responses' mean score gradients (Appendix~\ref{app:invariance}). Whether a learned shift lies mostly outside that direction is an empirical question, which \S\ref{sec:invisible} answers by measurement.

\paragraph{The common-mode offset.}
\label{sec:carrier-def}
Let $h_t$ and $h_t^{\text{base}}$ denote the final hidden states that the unembedding reads at position $t$ of a rejected response (evaluated under teacher forcing) for the trained and base models, respectively. We define the displacement as $\Delta h_t = h_t - h_t^{\text{base}}$. The \emph{common-mode offset} $\offb$ is the mean displacement over the response's positions, $\offb = \frac{1}{T}\sum_{t=1}^{T} \Delta h_t$, averaged again over a probe set of rejected responses. The \emph{residual} $\resr = \Delta h_t - \offb$ represents the position-specific part, such that $\Delta h_t = \offb + \resr$ holds exactly. This decomposition requires only two checkpoints and no access to training internals. We track the norm $\|\offb\|$, the cosine between the offsets of chosen and rejected responses, and the \emph{loudness} $\kappa \log \sum_v \exp\!\big((W \offb)_v / \kappa\big)$, a smooth maximum of its projection onto the vocabulary, where $W$ is the unembedding matrix and $\kappa$ a temperature.

\paragraph{Three stages of failure.}
\label{sec:def-collapse}
\emph{Likelihood displacement} is the routine erosion of absolute log-likelihoods for both chosen and rejected responses and, in isolation, is not pathological. \emph{Vocabulary well formation} is the intermediate state, in which the probability mass released by the displacement concentrates on one or a few tokens or token families; it can be measured before any text is sampled. \emph{Generation collapse} is the terminal state, in which these tokens dominate decoding. We count a run as \emph{collapsed} once it reaches either state, that is, once one token takes most of the released mass or most of its samples degenerate, by repeating text, ending nearly empty, or cycling through the forms of one word. Released mass can also spread thinly across the vocabulary without forming a well. We call the resulting loss of sample quality \emph{diffuse degradation}, measure it by the base model's per-token negative log-likelihood of the samples, and do not count it as collapse.

\paragraph{Measurement and configuration.}
\label{sec:measurement}
\label{sec:config}
\emph{Mass flow} is computed under teacher forcing on rejected responses. At every position we track the probability mass surrendered relative to the base model, the \emph{released mass}, the share of it absorbed by the top-1 token, the \emph{capture}, and the share leaking beyond the 10,000th-ranked vocabulary item, the \emph{tail}; a well has \emph{nucleated} once it captures the majority of the released mass. \emph{Degeneration} scores sampled generations for $n$-gram repetition \citep{Holtzman2019-ca,Welleck2019-lr}, together with a word-stem ring rate and the base model's per-token negative log-likelihood of the samples, which detects non-repeating degeneration that standard repetition counts overlook (Appendix~\ref{app:rulers}). \emph{Held-out ranking} is evaluated on 600 unseen lists, or on 300 of them for a few runs, by top-1 accuracy, pairwise accuracy, and the \emph{margin} between the chosen response's mean log-probability and that of the highest-ranked rejected response (Appendix~\ref{app:setup}). Unless stated otherwise, we fine-tune a rank-32 LoRA adapter with a learning rate of $10^{-5}$ and an effective batch size of 32; full-parameter runs use a learning rate of $10^{-6}$. A \emph{base budget} is three passes over the training data, which is 504 steps on the driving lists and about three quarters of one epoch of the official SimPO training data on UltraFeedback. Extended runs operate at fixed multiples of this budget under a schedule of the same form, and a base model's \emph{onset budget} is the smallest tested multiple at which its untreated run collapses. Offsets and mass flow use 80 rejected training responses unless stated otherwise.

\section{Ranking Improves While Generation Collapses}
\label{sec:what}
\label{sec:phenomenon}
\label{sec:generalization}

We first train Llama-2-7B \citep{Touvron2023-pi} with a LoRA adapter under sequential $K=4$ Plackett--Luce ranking on the driving lists. The lists were built for the Llama-2-based driving agent of \citet{Li2025-ka}. On 600 unseen lists, top-1 accuracy reaches 0.995, and pairwise accuracy stays above 0.94. Generation, however, fails. Within 504 steps, 89\% of the probability mass released relative to the base model concentrates on a single token, \texttt{textt}, which the base model ranked deep in its vocabulary tail, and almost every sampled generation degenerates into repeating it. By every ranking metric the objective has succeeded, yet the output distribution has collapsed.

We replicate this training across 17 base-model lineages from five tokenizer families, spanning 0.5B to 14B parameters. Fourteen of the 17 collapse within six times the base budget (Table~\ref{tab:crossfamily}; Appendix Table~\ref{tab:appendix-crossfamily}), including TinyLlama-1.1B and the Qwen2.5 models from 0.5B to 3B. Twelve form a single-token well, and in Qwen2.5-14B and Gemma-2-9B most samples cycle through the forms of one word, \texttt{collision} or \texttt{crash}. Phi-3-mini and Gemma-2-2B, tested only up to three times the budget, stay below this criterion, although 39\% of Phi-3-mini's samples repeat \texttt{speed}. The clear exception is Llama-3.2-3B. Its rejected responses fall by over 30 nat per token, yet its released mass stays on ordinary tokens and its samples stay nearly as fluent as the base model's (\S\ref{sec:geometry}).

\begin{table}[t]
\caption{Every base here but Llama-3.2-3B collapses, onto one token or the forms of one word. Untreated runs at their onset budget, or at $6\times$ for Llama-3.2-3B. Ring counts samples dominated by one word stem, on the 12 saved samples where marked $\dagger$ and unmeasured where marked --.}
\label{tab:crossfamily}
\centering
\scriptsize
\begin{tabular}{lcccccl}
\toprule
Base & Budget & $\Delta$chosen (nat) & Capture / tail & Degen. / ring & $\|\offb\|$ & Well \\
\midrule
Llama-2-7B     & $1\times$  & $-10.97$ & 89.4\% / 89.5\%  & 96.0\% / --      & 92.4  & \texttt{textt} \\
Mistral-7B-v0.3& $1\times$  & $-17.51$ & 95.1\% / 95.1\%    & 100\% / 0\%      & 300.4 & \texttt{/******/} \\
Qwen2.5-7B     & $3\times$ & $-21.67$ & 93.9\% / 86.7\%    & 99.3\% / --      & 328.7 & CJK fragment \\
Qwen2.5-14B    & $3\times$ & $-3.85$  & 27.0\% / 7.7\%     & 80.3\% / 78.7\%  & 198.6 & \texttt{collision}-fam. \\
Gemma-2-9B     & $1\times$  & $+0.62$  & 12.5\% / 37.5\%    & 0\% / 11 of 12$^\dagger$ & 163.7 & \texttt{crash}-fam. \\
Llama-3.1-8B   & $1\times$  & $-14.27$ & 87.6\% / 88.5\%    & 18.3\% / 22.0\%  & 115.2 & \texttt{.scalablytyped} \\
OLMo-2-7B      & $6\times$ & $-2.90$  & 89.6\% / 89.6\%    & 21.0\% / 8.0\%   & 87.8  & \texttt{contador} \\
OLMo-3-7B      & $3\times$ & $-6.83$  & 93.8\% / 93.8\%    & 96.0\% / 0\%     & 79.7  & (quote+4$\times$newline) \\
\midrule
Llama-3.2-3B   & $6\times$ & $+1.17$  & 43.7\% / 8.3\%     & 10.7\% / 7.7\%   & 66.7  & none (ordinary tokens) \\
\bottomrule
\end{tabular}
\end{table}

The collapse does not depend on adapters either. Full-parameter fine-tuning of Llama-2-7B at a learning rate of $10^{-6}$ collapses onto \texttt{textt} by step 250 of the same 504-step budget. Under full-parameter fine-tuning, onset tracks the learning rate summed over training, and from two earlier runs we wrote down before this run that it would collapse between steps 200 and 300 (Appendix~\ref{app:fullft}). The failure is therefore tied neither to one model lineage nor to adapters.

\section{The Common-Mode Offset Causes the Well}
\label{sec:mechanism}

\begin{figure}[t]
\centering
\includegraphics[width=0.88\textwidth]{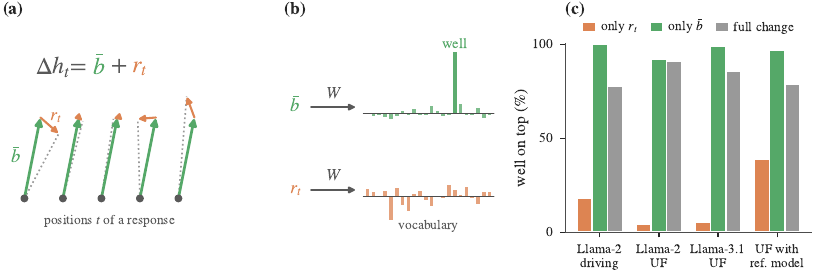}
\caption{(a) At every position of a rejected response, the shift of the final hidden state splits into the shared offset $\offb$ and a position-specific residual $\resr$. (b) Through the unembedding $W$, only $\offb$ singles out one vocabulary row. (c) Adding only $\offb$ to the base hidden state puts the well on top at nearly every rejected position, while adding only $\resr$ rarely does; UF denotes UltraFeedback.}
\label{fig:mechanism}
\end{figure}

\paragraph{The offset is shared across positions, prompts, and both responses.}
\label{sec:localizing}
We computed the offset separately for each of 900 rejected responses and compared each with their global mean. Across prompts, the per-response offsets align with their mean with median cosines above 0.97 on each of the six runs we measured. Across the two sides of a preference pair, the chosen and rejected offsets exhibit cosines of 0.83 to 0.996. The shift is therefore a single geometric translation that moves both sides of every pair together (Figure~\ref{fig:teaser}b). Furthermore, this shared vector points at the well, since for single-token wells the well token has the largest entry of the projected vector $W\offb$ (Figure~\ref{fig:mechanism}b).

\paragraph{The loss barely sees the offset.}
\label{sec:invisible}
We added the trained offset to the base model's hidden states and recomputed the scores of chosen and rejected responses on held-out lists. In the four collapsing runs we evaluated, the offset lowers both scores by 15 to 65 nat per token but moves the margin between them by only 0.2 to 2.7 nat in an average list (Appendix Table~\ref{tab:nullspace}). Its effect on both scores is 10 to 270 times its effect on their difference, and its cosine with the gradient of that difference in the base model is at most 0.16. The objective therefore observes a healthy margin, while the log-probabilities underneath fall by tens of nats. A $z$-loss \citep{Chowdhery2023-hc,Team-OLMo2024-ex} that compresses the log-partition function 2.5-fold still leaves 65 to 66\% of the released mass in the well (Appendix~\ref{app:zloss}). Little in the loss therefore holds the offset back, and fine-tuning alone already builds one. A supervised control without the ranking term grows an offset of norm 56 shared by chosen and rejected responses (cosine 0.99), yet it lifts no token by more than 3.4 (Appendix~\ref{app:sft}).

\paragraph{Keeping, hiding, adding, or removing the offset controls the well.}
\label{sec:causal}
We test the offset's causal role in four ways. The first adds $\offb$ alone or $\resr$ alone to the base hidden state. On untreated runs, $\offb$ alone makes the well the top token at nearly every rejected position, more often than the full trained state, while $\resr$ alone rarely does (Figure~\ref{fig:mechanism}c). The residual therefore partly works against the well. The same holds under full-parameter fine-tuning, where the fine-tuned unembedding rows alone almost never lift the well (Appendix~\ref{app:fullft}).

The second test hides the offset from the loss. \emph{Centering} computes the rejected-side ranking loss on $h_t - \text{stopgrad}(\offb)$. The rejected scores then no longer change when the offset grows, so pointing $\offb$ at one vocabulary row no longer lowers them. The gradient can still move the offset, since it treats the subtracted mean as a constant. On the Llama-2 flagship, centering keeps most of the offset's norm but turns its projection away from the vocabulary, and no well forms.

The third test transplants the trained offset into an untreated base model. This \emph{injection} adds $\offb$ to the final hidden state at every position, causing the well token to dominate generation on every held-out prompt. Stripping only the well's row component from $\offb$ removes that well, but generation instead collapses onto a different single token, whereas random or orthogonal directions of the same norm produce no well. The well therefore follows the offset rather than any particular token. The fourth test removes the offset from a model that has already collapsed. Subtracting half of its logit bias $W\offb$ from the untreated flagship model at decoding time lowers the share of degenerate greedy samples on 60 held-out prompts from 100\% to 0\%, and the chosen response regains 8.7 nat per token (Appendix~\ref{app:recovery}). Keeping only the offset reconstructs the well, hiding it prevents the well, adding it creates one, and removing it stops the collapse, making the offset a controllable cause of collapse.

\section{When the Offset Becomes a Well}
\label{sec:when}
\begin{figure}[t]
\centering
\includegraphics[width=0.88\textwidth]{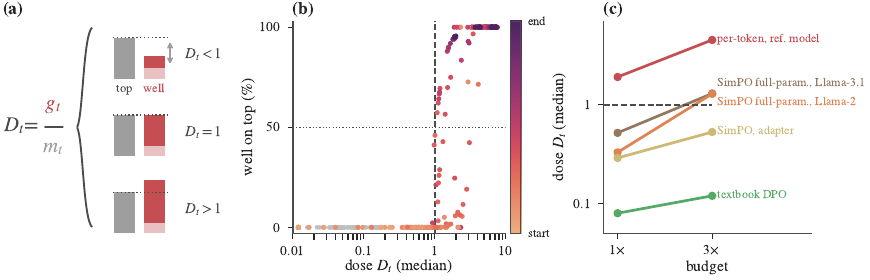}
\caption{(a) The offset overrides the base model at a position once its gain $g_t$ on the would-be well exceeds the base gap $m_t$, that is, once the dose $D_t$ exceeds one. (b) Across 510 snapshots of ten collapsing lineages, colored by training progress, and three summed-score runs that stay healthy, in gray, the well never tops a majority of rejected positions while the median dose stays below one. (c) Median doses of Table~\ref{tab:objectives} at the usual and the three-fold budget.}
\label{fig:dosemain}
\end{figure}

\paragraph{A well forms only after the offset's gain passes the base gap.}
\label{sec:dose}
The formation of a well depends on a logit ratio. Let $m_t$ be the \emph{base gap}, which measures how much the base model prefers its own top token over the candidate well token at position $t$, and let $g_t$ be the logit advantage that $\offb$ adds to that well token. In the offset-only state, the well overtakes the base model's top token where the \emph{dose} $D_t = g_t/m_t$ exceeds one (Figure~\ref{fig:dosemain}a). Scaling the injected offset confirms this threshold, since a quarter-strength $\offb$ fails to trigger the well while a half-strength one pushes $D_t$ above one and puts the well on top at nearly all rejected positions. We also tracked the dose along natural training runs (Figure~\ref{fig:dosemain}b; Appendix Figure~\ref{fig:dose} shows each lineage). Across all 510 snapshots, the well never tops a majority of rejected positions while the median dose stays below one, so in our untreated runs a well needs a dose above one, but a dose above one does not guarantee a well. In Llama-3.2-3B the median dose reaches 2.75 on the ordinary subword \texttt{e}. Yet the trained model puts \texttt{e} on top at only 29\% of rejected positions, against 80 and 83\% for the wells of Llama-3.1 and Llama-2, and its samples stay nearly as fluent as the base model's.

\paragraph{Wells tend to land on anomalous vocabulary rows.}
\label{sec:geometry}
Because the same $\offb$ projects more strongly onto rows with large norms, rare tokens with large unembedding rows gain the most. In the adapter runs on the driving data, most single-token wells sit in the top 1\% of their base model's row norms and never occur in WikiText-103 \citep{merity2017pointer}, whereas in Llama-3.2-3B the highest-norm rows belong to common subwords. High row norms are not required, however. On UltraFeedback Llama-2 collapses onto the ordinary token \texttt{Conf}, which the offset alone still puts on top at nearly every position (Figure~\ref{fig:mechanism}c). Row geometry makes some tokens more exposed, but in untreated runs the offset is what lifts them.

\paragraph{Different losses push the offset by very different amounts.}
\label{sec:demand}
We trained Llama-2-7B on the UltraFeedback pairs with our per-token ranking loss, the same loss with a reference model, SimPO, and textbook DPO, each at its usual setting and, except the plain per-token loss, at three times the budget. We also ran SimPO with full-parameter fine-tuning on Llama-2-7B and Llama-3.1-8B. Table~\ref{tab:objectives} and Appendix Figure~\ref{fig:std-recipes} report each run's demand, offset norm, dose, capture, share of degenerate samples, and win rate over the starting model under Skywork-Reward-V2 \citep{Liu2026-ht}.

Only the per-token ranking loss collapses at the usual budget, and adding a reference model does not stop it. SimPO forms no well at its usual setting, although its samples still degrade, and it passes the threshold with three times the budget under full-parameter fine-tuning on both Llama-2 and Llama-3.1 (Figure~\ref{fig:dosemain}c). Textbook DPO stays far below the threshold even with three times the budget, because its offset stays short and spreads across the vocabulary.

This difference stems from how far each loss keeps pushing. A ranking loss stops pushing once the chosen response beats the rejected one by a large enough margin, and we call the per-token size of that margin the loss's \emph{demand}.\footnote{A logistic ranking loss never stops pushing exactly, but once the scaled score margin reaches 5, its gradient weight has fallen to about 1\% of its value at zero margin. We take this point as the stopping margin, which gives a demand of about $5/\beta$ nat under per-token averaging and $5/(\beta|y|)$ nat under scores summed over a response of $|y|$ tokens, with $|y|=182$ on average in our UltraFeedback pairs. SimPO's target margin $\gamma$ adds $\gamma/\beta$ to its demand. Appendix Figure~\ref{fig:demand} shows the full sweep over $\beta$.} Our per-token loss keeps pushing until the gap reaches about 42 nat per token, well beyond the base model's own preference for its top token, so the dose passes one. SimPO stops at about 2.4 nat. Textbook DPO sums the score over the whole response, which spreads the same margin across hundreds of tokens and stops it at about 0.3 nat per token.

The reference model is therefore not what protects textbook DPO. Adding one to the per-token loss leaves a well in place on both datasets and at both budgets, with an offset as long as the untreated run's. What protects textbook DPO from the well is its low demand. Summing the score is roughly equivalent to averaging it with $\beta$ multiplied by the response length, so at the usual $\beta$ the protection comes from the resulting low demand.

\begin{table}[t]
\caption{Each loss at its usual setting and, where listed, with three times the budget, with doses above one in bold. Runs start from the base model unless marked SFT, and all train on the same UltraFeedback pairs. SimPO uses $\beta=2.5$ and $\gamma=1.0$ from base models and the official $\beta=2.0$ and $\gamma=1.0$ from the SFT checkpoint. Degen. counts repetitive or nearly empty samples, and Win compares a run's samples with those of the model it starts from.}
\label{tab:objectives}
\centering
\scriptsize
\begin{tabular}{llcccccc}
\toprule
Objective & Budget & Demand (nat) & $\|\offb\|$ & Dose & Capture & Degen. & Win \\
\midrule
Per-token ranking, $\beta=0.12$ & $1\times$ & 42 & 84 & \textbf{1.92} & 97\% & 99\% & 25\% \\
\quad with a reference model & $1\times$ & 42 & 79 & \textbf{1.91} & 86\% & 99\% & 26\% \\
\quad with a reference model & $3\times$ & 42 & 91 & \textbf{4.53}$^\dagger$ & 89\% & 99\% & 17\% \\
\midrule
SimPO & $1\times$ & 2.4 & 114 & 0.29 & 0.4\% & 0.7\% & 7\% \\
SimPO & $3\times$ & 2.4 & 82 & 0.53 & 9\% & 28\% & 6\% \\
SimPO, full-parameter & $1\times$ & 2.4 & 65 & 0.33 & 3\% & 0.7\% & 10\% \\
SimPO, full-parameter & $3\times$ & 2.4 & 85 & \textbf{1.29} & 27\% & 80\% & 7\% \\
\quad on Llama-3.1-8B & $1\times$ & 2.4 & 75 & 0.52 & 0.3\% & 30\% & 13\% \\
\quad on Llama-3.1-8B & $3\times$ & 2.4 & 101 & \textbf{1.30}$^\dagger$ & 45\% & 86\% & 3\% \\
\quad from SFT, on Llama-3-8B & $1\times$ & 3.0 & 33 & 0.15 & 2.5\% & 1.3\% & 63\% \\
\quad from SFT, on Llama-3-8B & $3\times$ & 3.0 & 77 & 0.19 & 2.7\% & 0\% & 34\% \\
\midrule
Textbook DPO, $\beta=0.1$ & $1\times$ & 0.3 & 14 & 0.08 & 4\% & 3\% & 68\% \\
Textbook DPO, $\beta=0.1$ & $3\times$ & 0.3 & 44 & 0.12 & 4\% & 1\% & 66\% \\
\bottomrule
\end{tabular}
\par\vspace{2pt}{\scriptsize $^\dagger$Computed for the run's own leading token, \texttt{D} or \texttt{Apache}; all other doses are for \texttt{Conf}.}
\end{table}

Between these extremes, SimPO within its published range grows an offset about as long as the collapsing run's, but a quiet one whose dose remains below one, so no well forms. Nevertheless, its generation quality degrades. The released mass spreads deep into the vocabulary instead of settling on one token, and the samples devolve into strings of unrelated words that standard repetition counts fail to detect, while the reward model prefers the base model's sample on about nine in ten prompts. However, adding the SimPO offset alone to the base model does not reproduce this diffuse form (Appendix~\ref{app:rm}). This diffuse degradation is therefore a second route to damage, separate from the well. The missing fine-tuning stage alone does not explain this damage, since textbook DPO from the same base model wins against it on about two thirds of prompts.

Under a three-fold budget, however, the offset grows louder, concentrating its projection onto a narrow set of vocabulary rows. Its dose keeps rising, and under full-parameter fine-tuning it passes one on both Llama-2 and Llama-3.1. By the end of training, 80 to 86\% of samples degenerate, and held-out ranking is the best of any UltraFeedback run on that base.

We also trained SimPO from its authors' public SFT checkpoint of Llama-3-8B with their hyperparameters (Appendix~\ref{app:sftstart}). From this start, no well forms within three times the base budget, and the median dose of every token stays below 0.4. The offset nonetheless grows from a norm of 33 at the base budget to 77 at three times the budget. It stays quiet, lifting no token by more than 4.3 nat, against about 18 nat for SimPO from the base models at the same budget. Meanwhile the SFT model's per-token negative log-likelihood of the samples rises from 1.35 to 3.86. Two public reward models agree, preferring the one-fold run's samples to the SFT model's on 63 to 65\% of held-out prompts but the three-fold run's on only 33 to 34\%. Supervised fine-tuning enlarges the base gaps of the candidate wells by at most 22\%, so by the dose it should delay the well rather than prevent it.

Supervised fine-tuning already produces an offset of the same form, and the preference loss sets where it points and how far it grows (Appendix~\ref{app:sft}). Small $\beta$ and listwise objectives therefore do not invent a new failure mode; they push harder on a shift that fine-tuning already produces.

\section{Controlling the Offset}
\label{sec:recipe}
\label{sec:recipe-defn}
Centering hides the offset from the loss but leaves most of its norm within the model. Our recipe therefore keeps centering and adds two penalty terms that act on the offset itself, used alone or together, and we call the combination \emph{both}. A norm penalty $\lambda_{\text{norm}} \|\offb\|$ shrinks the offset, and a loudness penalty $\lambda_{\text{loud}} \cdot \kappa \log \sum_v \exp\!\big((W\offb)_v/\kappa\big)$ flattens its vocabulary projection without necessarily shrinking its norm. On the driving lists the recipe also keeps the NLL term of weight 1.0, so it differs from the NLL-only baseline only by offset control. Offset control matters most for objectives worth keeping despite a high demand. Listwise ranking, for example, also supervises the order among the rejected responses, which pairwise objectives never see, and on held-out driving lists it keeps that order at 0.82 to 0.92 while pairwise ranking inverts it (Appendix~\ref{app:listwise}).

\paragraph{Cross-family validation up to three times the base budget.}
\label{sec:recipe-validation}
We applied the recipe to four base models at or beyond the budget where each untreated run collapses, with Llama-2, Qwen2.5-7B, and Llama-3.1 at three times the base budget (Table~\ref{tab:recipe}). No run forms a well. The largest share of the released mass, at most 20\%, goes to an ordinary token such as the period, the tail stays at or below 5.5\%, and sampled degeneration remains at or below about 12\%. By contrast, at the same three-fold budget on Llama-2, the NLL-only baseline, N-DPO, and DPO-Positive leave 71 to 81\% of the released mass in a single well (Table~\ref{tab:recipe}). The cost to ranking performance is small, since held-out pairwise accuracy settles between 0.94 and 0.96. On Llama-2, Mistral, and Qwen2.5-7B, this is within 2 points of a run of the same base and budget without offset control. Controlling the offset is therefore enough to prevent collapse up to three times the base budget, at almost no cost to ranking.\label{sec:recipe-limits} At six times the budget on OLMo, the residual, not $\offb$, lifts the surviving well, and a residual penalty closes the remaining routes on OLMo-3 (Appendix~\ref{app:extended}).

\begin{table}[t]
\caption{The recipe, centering plus a norm penalty, a loudness penalty, or both, at or beyond each base's onset budget, where capture goes to an ordinary token. Three Llama-2 baselines without offset control keep 71 to 81\% of the released mass on one rare token.}
\label{tab:recipe}
\centering
\scriptsize
\begin{tabular}{lccccc}
\toprule
Base & Penalty & Steps & Capture / tail & Degen. & Held-out top-1 / pairwise \\
\midrule
Llama-2-7B     & loudness & 1512 & 19.8\% / 1.5\%  & 6--12\% & 1.00 / 0.95 \\
Mistral-7B     & both     & 504  & 14.6\% / 1.4\%   & 3.0\%  & 1.00 / 0.94 \\
Qwen2.5-7B     & both     & 1512 & \phantom{0}8.0\% / 0.2\%   & 1.7\%  & 1.00 / 0.96 \\
Llama-3.1-8B   & both     & 1512 & 13.0\% / 5.5\%   & 6.7\%  & 1.00 / 0.96 \\
\midrule
\multicolumn{6}{l}{\emph{Llama-2-7B without offset control}} \\
NLL only       & --       & 1512 & 70.6\% / 71.0\%  & 2.0\%  & 1.00 / 0.97 \\
N-DPO          & --       & 1512 & 78.1\% / 78.4\%  & 18.0\% & 1.00 / 0.96 \\
DPO-Positive   & --       & 1512 & 81.4\% / 82.2\%  & 24.3\% & 1.00 / 0.96 \\
\bottomrule
\end{tabular}
\end{table}

\section{What Ranking Metrics Miss}
\label{sec:dissociation}

\paragraph{An external judge disagrees with the ranking metrics.}
\label{sec:judge}
Held-out pairwise accuracy levels off at 0.93 to 0.97 within 300 to 500 steps, treated or not, while the margin keeps growing. The untreated run with the largest held-out margin typically collapses, in line with the over-optimization of direct alignment algorithms \citep{Rafailov2024-bz} (Appendix Figure~\ref{fig:margin-app}). We evaluated one sample for each of 300 held-out UltraFeedback prompts using two public reward models, Skywork-Reward-V2 \citep{Liu2026-ht} and ArmoRM \citep{Wang2024-kb}, across 17 Llama-2 runs (Appendix Table~\ref{tab:rm} and Figure~\ref{fig:margin}a). Across these runs, the better a run's held-out pairwise accuracy, the less often the external judge prefers its sample to the base model's, with a Spearman correlation of $-0.78$ for Skywork-Reward-V2 and $-0.79$ for ArmoRM. The runs that rank best move the most probability away from the base model's likely tokens, and the judge tracks how much stays there. The share of mass remaining on the base model's ten most likely tokens correlates with the win rate at $+0.74$ and $+0.87$. Since the base models of these runs are not instruction-tuned, we use the win rate only to compare runs. By this measure, the recipe with an NLL term of weight 1.0 matches textbook DPO, winning 65\% of prompts against 66 to 68\%, as does the NLL-only baseline at 64\%.

\paragraph{The offset warns before the well forms.}
\label{sec:early}
In seven of the nine lineages with snapshots every 50 steps, the eventual well enters the eight largest entries of $W\offb$ before it takes over, up to 1,200 steps ahead, even while held-out ranking has long saturated (Appendix Figure~\ref{fig:margin}b). Tracking $\|\offb\|$, its loudness, and the leading entries of $W\offb$ during training therefore offers an early signal that standard ranking metrics do not give, although we have read it only from completed runs.

\section{Conclusion}
\label{sec:conclusion}
Preference optimization can collapse generation while held-out ranking stays high. We traced this failure to a common-mode offset, a shift of the final hidden state shared across positions and across chosen and rejected responses, which the ranking objective barely penalizes. A well forms only after the offset lifts one token above the tokens the base model would otherwise predict. Different ranking losses turn the same preference pressure into very different doses. Hiding the offset from the loss and penalizing it prevents collapse up to three times the base budget. Across runs, held-out ranking accuracy even anticorrelates with an external judge ($\rho=-0.78$). Ranking accuracy and preference margin measure only what the loss constrains, so the offset deserves monitoring.

\paragraph{Limitations.}
\label{sec:discussion}
Our evidence spans 17 base-model lineages, two datasets, and adapter and full-parameter training, with two further seeds for the flagship settings and one seed elsewhere (Appendix~\ref{app:scope}). The penalty strengths are calibrated on the driving data, the six-fold result covers the OLMo family only, and apart from the SFT-start runs the demand analysis starts from base models.

\section*{Reproducibility Statement}
Training configurations (base models, LoRA rank, learning rate, batch composition, and the form of every penalty term), the exposure-ordering and offset-decomposition measurement code, and the held-out ranking probe are described in \S\ref{sec:setup}--\S\ref{sec:recipe} and Appendix~\ref{app:details}; code with all penalty coefficients and the training and held-out preference data will be released upon publication. Held-out ranking numbers are computed on 600 lists, or on 300 of them where stated, whose prompts and chosen responses do not appear in any training split; the corpus counts in \S\ref{sec:geometry} use the full training text and the WikiText-103 training split.

\section*{AI Use Statement}
The authors set the research questions, made every final design decision, and checked every reported number against the logged experimental runs and every citation against its source. Within this process, we used generative AI tools (Claude Code and Gemini) to help develop the conceptual framework, propose and refine hypotheses, design experiments and give feedback on the methodology, formulate mathematical claims and assist in writing their proofs, implement the methods, clean and reformat datasets, interpret results, and assist with translation. We did not use these tools to generate synthetic datasets. We have reviewed all AI-assisted work and take responsibility for the final content of this work, including text, claims or artifacts produced with the aid of generative AI.

\subsubsection*{Acknowledgments}
This work was supported by JSPS KAKENHI Grant Number JP26KJ0799 and by TIER IV, Inc.

\bibliography{references}
\bibliographystyle{iclr2027_conference}

\clearpage
\appendix
\section*{Appendix Overview}
Appendix~\ref{app:related} extends the related work of \S\ref{sec:related}. Appendix~\ref{app:details} collects the setup details, full results, and supporting measurements behind the main text, and ends with the formal statement of the common-shift invariance. The third column names the main-text section that each part supports.

\begin{center}
\small
\begin{tabular}{@{}l p{0.70\textwidth} l r@{}}
\toprule
Part & Contents & Supports & Page \\
\midrule
\ref{app:related} & \textit{Extended related work.} Iterated preference optimization, likelihood displacement, likelihood regularizers, and representation geometry. & \S\ref{sec:related} & \pageref{app:related} \\
\ref{app:onset} & \textit{Full cross-family onset table.} All 17 lineages, the budget at which each collapses if it does, and the collapse criterion. & \S\ref{sec:generalization} & \pageref{app:onset} \\
\ref{app:setup} & \textit{Setup and measurement details.} $\beta$ per scene type, exposure, and the full definitions of every metric. & \S\ref{sec:setup} & \pageref{app:setup} \\
\ref{app:curves} & \textit{Training curves and extended-budget trajectories.} Flagship Llama-2 runs and six-fold-budget OLMo runs. & \S\ref{sec:recipe}, \S\ref{sec:dissociation} & \pageref{app:curves} \\
\ref{app:extended} & \textit{Routes of the released mass at extended budget.} Where the mass goes under each penalty, and the residual penalty. & \S\ref{sec:recipe} & \pageref{app:extended} \\
\ref{app:penalty} & \textit{Penalty trajectories across families.} Offset norm and its largest vocabulary entry under each penalty. & \S\ref{sec:recipe} & \pageref{app:penalty} \\
\ref{app:listwise} & \textit{A listwise advantage orthogonal to collapse.} Ordering of unseen rejected responses by objective. & \S\ref{sec:recipe} & \pageref{app:listwise} \\
\ref{app:recovery} & \textit{Removing the offset from a collapsed model.} Subtracting its logit bias at decoding time undoes the collapse. & \S\ref{sec:causal} & \pageref{app:recovery} \\
\ref{app:zloss} & \textit{Absolute logit scale.} A $z$-loss compresses $\log Z$ without removing the well. & \S\ref{sec:mechanism} & \pageref{app:zloss} \\
\ref{app:dose} & \textit{Dose along training and the objective's demand.} Dose per lineage, the sweep over $\beta$, and standard-recipe offsets. & \S\ref{sec:when} & \pageref{app:dose} \\
\ref{app:fullft} & \textit{Full-parameter onset.} Onset tracks the integrated learning rate. & \S\ref{sec:generalization}, \S\ref{sec:mechanism} & \pageref{app:fullft} \\
\ref{app:sft} & \textit{A supervised fine-tuning control.} Plain fine-tuning builds a shared offset but no well. & \S\ref{sec:mechanism}, \S\ref{sec:when} & \pageref{app:sft} \\
\ref{app:sftstart} & \textit{SimPO from a supervised fine-tuning checkpoint.} The official recipe from a public SFT model, one-fold and three-fold budgets. & \S\ref{sec:demand} & \pageref{app:sftstart} \\
\ref{app:rm} & \textit{External reward models.} Judge protocol, win rates of all runs, the early-signal panel, and ranking against the margin. & \S\ref{sec:when}, \S\ref{sec:dissociation} & \pageref{app:rm} \\
\ref{app:rulers} & \textit{Shapes of well that standard rulers miss.} Latent wells, non-repeating wells, and word-family rings. & \S\ref{sec:setup}, \S\ref{sec:when} & \pageref{app:rulers} \\
\ref{app:supp} & \textit{Supplementary measurements.} Further numbers behind the claims of the main text. & \S\ref{sec:mechanism}--\S\ref{sec:dissociation} & \pageref{app:supp} \\
\ref{app:scope} & \textit{Seeds and scope.} Seed repeats, adapter ranks, and the ranges of the demand analysis. & \S\ref{sec:conclusion} & \pageref{app:scope} \\
\ref{app:invariance} & \textit{Common-shift invariance and the trained offset.} Two propositions and the measured split of the offset. & \S\ref{sec:setup}, \S\ref{sec:mechanism} & \pageref{app:invariance} \\
\bottomrule
\end{tabular}
\end{center}

\section{Extended Related Work}
\label{app:related}

\textbf{Iterated preference optimization.} Recent pipelines repeat preference optimization over several rounds on continually refreshed preference data \citep{Yuan2024-mp,Grattafiori2024-rc}, although sampling more responses per prompt does not always help \citep{Xiao2025-jp}.

\textbf{Likelihood displacement, squeezing, and gradient entanglement.} Recent studies have documented the probabilistic symptoms of preference optimization, but from different angles. \citet{razin2025unintentional} demonstrate that DPO training can decrease the chosen response's absolute likelihood. They trace this to the unembedding matrix's geometry, using a centered hidden-embedding similarity score to predict \emph{when} a dataset risks spilling rejected updates onto chosen ones. \citet{yuan2025common} prove that margin-based losses entangle the gradients of chosen and rejected log-probabilities, explaining \emph{why} the two tend to move in tandem. While these works map the symptoms in probability space, we isolate in representation space the \emph{common-mode offset} of the final hidden state, which produces the well. Furthermore, \citet{ren2025learning} show that the negative gradient on rejected tokens ``squeezes'' probability mass toward tokens the base model is already confident about. Pushing down the probability of unwanted tokens is also how unlikelihood training reduces repetition \citep{Welleck2019-lr}, whereas in preference optimization the push on rejected responses can end in repetition. The vocabulary wells we discover represent a different phenomenon, since the base model initially ranks our well tokens far below its preferred tokens, typically by 17 to 39 nat. It is the accumulation of the common-mode offset that inflates the logits of these low-confidence tokens, until one of them overcomes this gap and dominates decoding at most positions.

\textbf{Regularizing absolute likelihood.} Several methods anchor the absolute probabilities that a relative ranking loss leaves free. DPO-Positive \citep{pal2024smaug} penalizes the chosen log-probability if it falls below its reference value; \citet{pang2024iterative} inject an explicit NLL term for reasoning tasks; and \citet{im2026normalized} regularize the length-normalized log-probability of the full response set toward the reference. However, because these interventions constrain the scalar probabilities rather than the underlying vector representations, they do not directly constrain the shared direction in the hidden state. In our experiments, an NLL term on the chosen response delays the growth of the common-mode offset without eliminating it (Figure~\ref{fig:olmo6x}). When it does prevent immediate collapse, the vocabulary well relocates to positions \emph{after} the desired answer, where the NLL term exerts no constraint.

\textbf{Unembedding geometry and representation anisotropy.} Prior work has mapped the geometry of pretrained representations. Word embeddings and contextual hidden states are anisotropic, sharing a common mean and a few dominant directions \citep{Mu2018-xl,Ethayarajh2019-gv,Bis2021-tl}, and standard training often squeezes them into a narrow cone \citep{Gao2019-sp}. Furthermore, specific outlier dimensions are known to track token frequency \citep{Puccetti2022-ri,macocco2025outlier}, encode massive activations \citep{sun2024massive}, or regulate output entropy \citep{stolfo2024confidence}. However, these works describe the geometry of a \emph{pretrained} model. The common-mode offset we study is constructed during preference \emph{training}. It does not reshape the entire embedding space into a cone; instead, it acts as a single, shared displacement vector added to the hidden states, which increases the inner product with one specific vocabulary row. Functionally, this resembles activation steering, where adding a fixed vector at every position predictably alters a model's behavior \citep{Rimsky2024-to}. Mechanistic studies of DPO also find that it changes behavior by shifting the residual stream, bypassing rather than removing capabilities the model already has \citep{Lee2024-df}. The key distinction is that preference optimization grows this steering vector on its own in the background, and the behavior it ultimately produces is generation collapse. Finally, \citet{land2024fishing} identify ``under-trained'' tokens with near-zero frequencies in base models. While many of the severe vocabulary wells in our experiments happen to emerge from this vulnerable population, these fragile tokens do not spontaneously cause collapse; they merely provide a path of least resistance for the common-mode offset. Decoding-time corrections that subtract unconditional logit terms \citep{ding2026breaking} act after the fact and leave the training-time cause in place.

\section{Additional Details}
\label{app:details}

\subsection{Full Cross-Family Onset Table}
\label{app:onset}
Table~\ref{tab:appendix-crossfamily} reports every base-model lineage with at least one untreated run.

\begin{table}[h]
\caption{Fourteen of the 17 base-model lineages collapse within six times the base budget (\S\ref{sec:generalization}). A run collapses once one token takes most of the released mass, a well, or once most of its samples degenerate or ring on one word stem, a word ring. Each row shows the untreated run at the onset budget, or at the largest budget tested when no collapse occurs. The onset multiple is relative to the 504-step base budget under the fixed schedule. TinyLlama, the three smallest Qwen2.5 models, and Llama-3.2-3B were trained at one and six times the budget only, Qwen2.5-14B first at three times, Phi-3-mini and Gemma-2-2B at one and three times, and Gemma-2-9B at the base budget only. Ring rates marked -- come from audits that predate the ring measure, and $\dagger$ marks a count on the 12 saved samples. Each family is cited at its first row.}
\label{tab:appendix-crossfamily}
\centering
\scriptsize
\begin{tabular}{llcccc}
\toprule
Base & Vocabulary & Onset multiple & Capture & Degen. / ring & Collapse \\
\midrule
Llama-2-7B \citep{Touvron2023-pi} & 32k & $1\times$ & 89.4\% & 96.0\% / -- & well \\
LLaMA-1-7B \citep{Touvron2023-fh} & 32k & $1\times$ & 53.0\% & 86.3\% / -- & well \\
TinyLlama-1.1B \citep{Zhang2024-fp} & 32k & $6\times$ & 99.0\% & 100\% / 100\% & well \\
Llama-3.2-1B-Instruct \citep{Grattafiori2024-rc} & 128k & $1\times$ & 88.8\% & 61.0\% / -- & well \\
Llama-3.2-3B & 128k & none within $6\times$ & 43.7\% & 10.7\% / 7.7\% & no \\
Llama-3.1-8B & 128k & $1\times$ & 87.6\% & 18.3\% / 22.0\% & well \\
Qwen2.5-0.5B \citep{Yang2024-rt} & 152k & $6\times$ & 93.9\% & 7.7\% / 6.7\% & well \\
Qwen2.5-1.5B & 152k & $6\times$ & 93.8\% & 99.7\% / 1.3\% & well \\
Qwen2.5-3B & 152k & $6\times$ & 89.6\% & 91.3\% / 0.3\% & well \\
Qwen2.5-7B & 152k & $3\times$ & 93.9\% & 99.3\% / -- & well \\
Qwen2.5-14B & 152k & $3\times$ & 27.0\% & 80.3\% / 78.7\% & word ring \\
Mistral-7B-v0.3 \citep{Jiang2023-mr} & 32k & $1\times$ & 95.1\% & 100\% / 0\% & well \\
OLMo-2-7B \citep{Team-OLMo2024-ex} & 100k & $6\times$ & 89.6\% & 21.0\% / 8.0\% & well \\
OLMo-3-7B \citep{Team-Olmo:-Allyson-Ettinger2025-bw} & 100k & $3\times$ & 93.8\% & 96.0\% / 0\% & well \\
Gemma-2-2B \citep{Gemma-Team:-Morgane-Riviere2024-av} & 256k & none within $3\times$ & 0.4\% & 4.7\% / 4.3\% & no \\
Gemma-2-9B & 256k & $1\times$ & 12.5\% & 0\% / 11 of 12$^\dagger$ & word ring \\
Phi-3-mini \citep{Abdin2024-xb} & 32k & none within $3\times$ & 18.4\% & 15.7\% / 38.7\% & no \\
\bottomrule
\end{tabular}
\end{table}

\subsection{Setup and Measurement Details}
\label{app:setup}
On the driving lists, $\beta$ is fixed per scene type, 0.12 for the two thirds of lists in normal driving and 0.18 to 0.40 for the other scene types; on UltraFeedback, $\beta = 0.12$. SimPO's target margin $\gamma$ is subtracted from the top response's scaled score at each stage. Because a list of length $K$ yields $K-1$ comparisons, the response ranked $k$-th is rejected in $k-1$ of them, a quantity we call its \emph{exposure}. In the flagship run the rejected responses fall in order of their exposure, so the more often a response serves as a negative baseline, the steeper its drop.

We evaluate each trained run along four axes, since no single metric captures the whole failure. \emph{Mass flow} is defined in \S\ref{sec:measurement}. \emph{Degeneration} combines standard metrics for $n$-gram repetition \citep{Holtzman2019-ca,Welleck2019-lr}, non-ASCII runs, and abnormal length with two measures of our own. A word-stem ring rate catches token-family wells, which we call \emph{word-family rings}, and the base model's per-token negative log-likelihood of the samples catches degeneration that does not repeat. \emph{Held-out ranking} uses 600 unseen lists of a prompt and its responses, or 300 of them for the Mistral runs, the six-fold OLMo-2 runs, and the training trajectories. Top-1 accuracy is the fraction of lists whose chosen response scores highest, and pairwise accuracy is the fraction of response pairs within a list that the model orders correctly. Kendall's $\tau$ \citep{Kendall1938-uf} is computed over the unseen rejected responses, and the margin is the chosen response's mean log-probability minus that of the highest-ranked rejected response. \emph{Offset geometry} covers the norm $\|\offb\|$, the loudness, and, for candidate well tokens, their unembedding row-norm percentile and their counts in the training text and in the $10^8$-token WikiText-103 corpus \citep{merity2017pointer}. For centering, $\offb$ is computed on the current batch against the adapter-free forward pass. Collapse and ranking dynamics are measured across all lineages, while the computationally intensive analyses of \S\ref{sec:mechanism} and \S\ref{sec:when} use the representative subsets named there.

\subsection{Training Curves and Extended-Budget Trajectories}
\label{app:curves}
The runs without offset control drive the ranking loss lowest and the training margin highest, and almost all of their extra margin comes from pushing the rejected responses down (Figure~\ref{fig:curves}). Figure~\ref{fig:olmo6x} shows the six-fold-budget OLMo trajectories behind \S\ref{sec:recipe-limits}, including the per-position loudness penalty runs.

\begin{figure}[h]
\centering
\includegraphics[width=\textwidth]{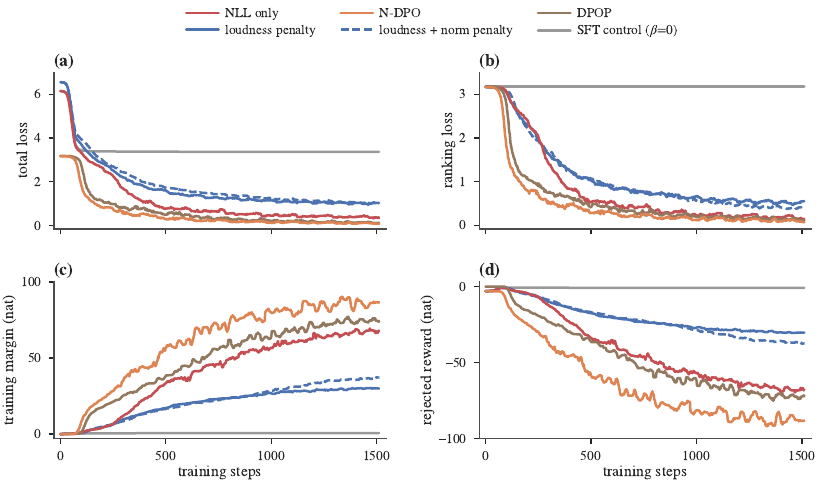}
\caption{Training curves of six Llama-2-7B runs with $K{=}4$ sequential PL on the driving lists at three times the base budget, smoothed with a 25-step running median. (a) Total loss and (b) ranking loss fall lowest in the three runs without offset control, NLL only, N-DPO, and DPO-Positive (DPOP), which end with 71 to 81\% of the released mass in a single well. The ranking loss of the SFT control ($\beta{=}0$) stays at its initial value. (c) Without offset control the training margin reaches 67 to 87 nat, against 30 and 37 nat with it, and (d) almost all of this difference comes from pushing the rejected responses down.}
\label{fig:curves}
\end{figure}

\begin{figure}[h]
\centering
\includegraphics[width=\textwidth]{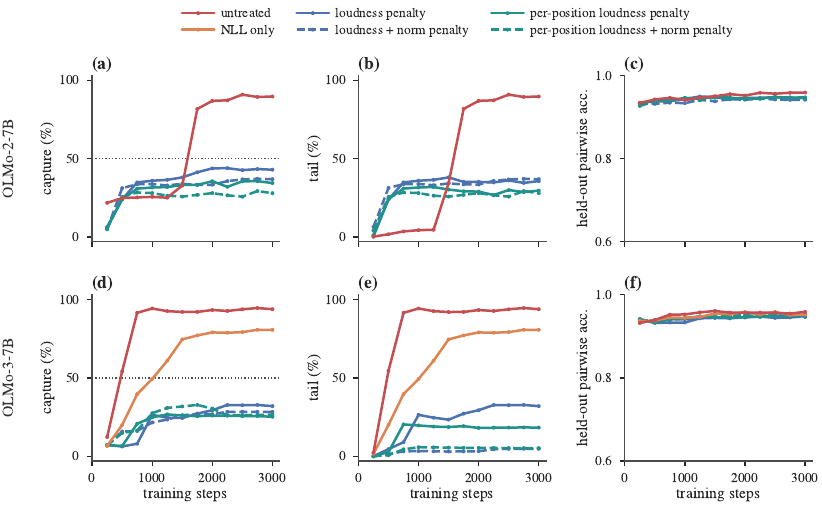}
\caption{Six-fold-budget training trajectories of OLMo-2-7B (a--c) and OLMo-3-7B (d--f) in six conditions, namely untreated, NLL only, and four offset-control variants that penalize loudness, loudness and norm, per-position loudness, or per-position loudness and norm. OLMo-2 has no NLL-only run. (a) and (d) Well capture crosses the dotted 50\% line in the untreated and NLL-only runs and ends at 81 to 94\%, while no offset-control variant exceeds 45\%. (b) and (e) The tail beyond rank 10{,}000 rises with the capture in those runs. With offset control it ends at 28 to 37\% on OLMo-2, but near 5\% on OLMo-3 once the norm penalty is added. (c) and (f) Pairwise accuracy on 300 held-out lists stays between 0.93 and 0.96 in every run.}
\label{fig:olmo6x}
\end{figure}

\subsection{Routes of the Released Mass at Extended Budget}
\label{app:extended}
At six times the base budget, each penalty on the OLMo family closes one route for the released mass, and the mass takes the next one (Figure~\ref{fig:olmo6x}). Penalizing loudness alone moves the well to another token of the same token family. Adding the norm penalty removes the well but sends about half of the released mass to punctuation and premature end-of-text tokens. Applying the loudness penalty at every position, rather than only to the mean offset, does not close this route either. Combined with the norm penalty, it still leaves about 40\% of the mass on punctuation and end-of-text on OLMo-3, and without the norm penalty it turns most OLMo-2 samples into word-family rings. In the treated runs that still leak, the base state plus $\offb$ never makes the well the top token, while the base state plus $\resr$ does so at 11 to 23\% of positions, so the residual has become the component that lifts the surviving well. Penalizing the residual itself, with $\lambda_{\text{res}} \frac{1}{T}\sum_t \|\resr\|$, closes both routes on OLMo-3 and leaves no well, no tail, and unchanged held-out ranking. On OLMo-2, however, the mass moves on into word-family rings and the deep tail.

\subsection{Penalty Trajectories Across Families}
\label{app:penalty}
Figure~\ref{fig:lenloud} shows how each penalty shapes the offset over training on Llama-2 and Qwen. The norm penalty keeps the offset short, while the loudness penalty lets it grow but keeps its largest vocabulary entry small. On Qwen, the loudness penalty alone lets the offset grow to 490 to 550, more than ten times its length under the norm penalty, which is consistent with the family-specific penalty choice in Table~\ref{tab:recipe}.

\begin{figure}[h]
\centering
\includegraphics[width=\textwidth]{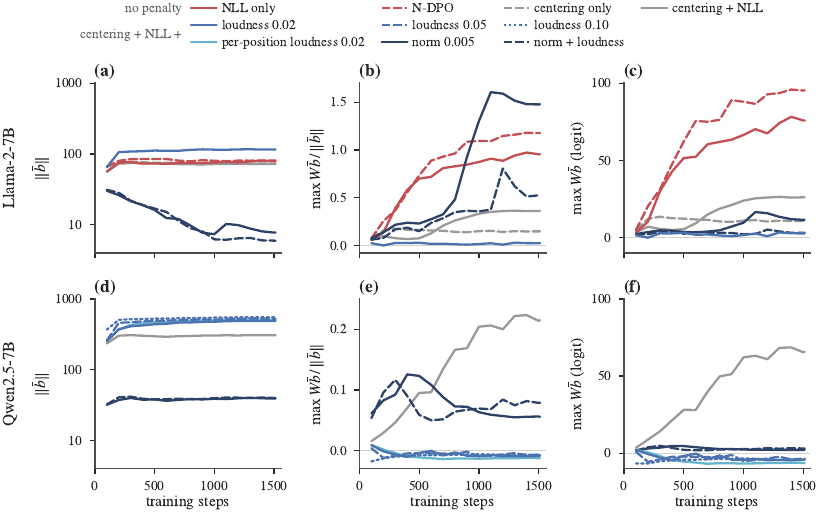}
\caption{(a) Across seven Llama-2-7B runs, $\|\offb\|$ stays between 70 and 85 without a penalty, grows to 116 under the loudness penalty, and falls to 6 to 8 once the norm penalty is added. (b) Divided by $\|\offb\|$, the largest entry of $W\offb$ stays near zero only under the loudness penalty. (c) NLL only and N-DPO lift the largest entry of $W\offb$ to 76 and 95 by the end of training, while the loudness penalty holds it below 4. (d) Across seven Qwen2.5-7B runs, $\|\offb\|$ grows to 490 to 550 under every loudness penalty but stays near 40 once the norm penalty is added. (e) Divided by $\|\offb\|$, the largest entry of $W\offb$ falls below zero under every loudness penalty. (f) Without a penalty, the largest entry of $W\offb$ climbs to 66, while every penalized run stays below 5.}
\label{fig:lenloud}
\end{figure}

\subsection{A Listwise Advantage Orthogonal to Collapse}
\label{app:listwise}
\begin{figure}[h]
\centering
\includegraphics[width=\textwidth]{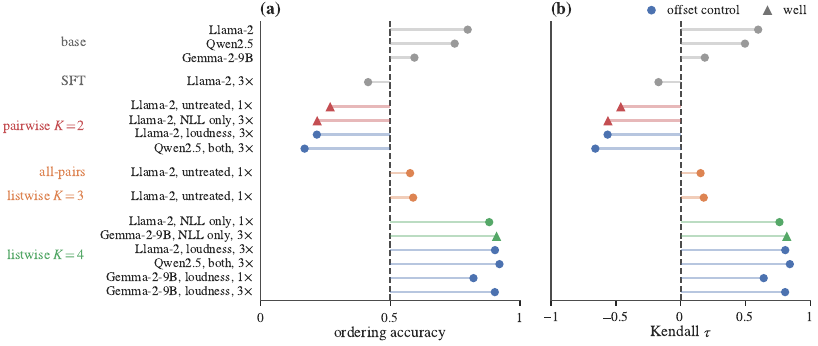}
\caption{(a) On 600 held-out lists, pairwise $K{=}2$ training orders the unseen rejected responses below chance (dashed line), with or without offset control (blue), while every sequential listwise $K{=}4$ run orders them better than the base models. (b) The Kendall $\tau$ over the rejected responses alone is negative for SFT and every pairwise run, and at least 0.64 for every listwise $K{=}4$ run. Triangles mark runs that form a well.}
\label{fig:rr}
\end{figure}
Held-out evaluation reveals one respect in which listwise and pairwise objectives differ, independent of collapse. Beyond whether the single best response is identified, we ask whether a model preserves the graded ordering among the \emph{unseen rejected} responses of a held-out list. Pairwise training never supervises this ordering, since it only compares the top response against one alternative. On the driving lists, pairwise ranking ($K{=}2$) inverts the graded ordering, with ordering accuracy of 0.17 to 0.27 and Kendall $\tau$ of $-0.46$ to $-0.66$, while every sequential listwise $K{=}4$ run preserves it, at 0.82 to 0.92 (Figure~\ref{fig:rr}). All-pairs and $K{=}3$ training fall in between at about 0.58, and SFT alone falls slightly below chance, at 0.415. On UltraFeedback the effect does not appear, and pairwise training instead slightly improves the ordering of rejected responses over the base model. This listwise benefit is therefore domain-dependent and orthogonal to the collapse mechanism of \S\ref{sec:mechanism}. It does not correlate with whether a run nucleates a well, and offset-controlled runs keep whichever ordering their underlying objective produces.

\subsection{Removing the Offset from a Collapsed Model}
\label{app:recovery}
We subtract a multiple of the logit bias $W\offb$ from the untreated flagship model at decoding time, with $\offb$ measured on 80 rejected responses, and decode 60 held-out prompts greedily for 100 tokens. Subtracting $\offb$ from the final hidden state gives the same result, since the unembedding is linear. Removing half of the bias stops the collapse, and removing all of it works nearly as well (Table~\ref{tab:recovery}). Removing only the well's own logit is not enough, because the collapse then moves to the next token of the same direction, a rare Cyrillic token. The recovered model continues text normally but does not produce the answer format, which the untreated run never learned, so removing the offset stops the damage without aligning the model.

\begin{table}[h]
\caption{Subtracting the offset's logit bias from the collapsed flagship model at decoding time. Degen. is the share of degenerate greedy samples, and the log-probabilities are per-token means under teacher forcing with the same subtraction.}
\label{tab:recovery}
\centering
\scriptsize
\begin{tabular}{lccc}
\toprule
Subtracted & Degen. & Chosen log-prob (nat) & Rejected log-prob (nat) \\
\midrule
Nothing & 100\% & $-13.37$ & $-65.01$ \\
Only the well's logit & 22\% & $-8.73$ & $-54.64$ \\
The eight largest entries of $W\offb$ & 65\% & $-7.17$ & $-48.76$ \\
Half of $W\offb$ & 0\% & $-4.70$ & $-39.44$ \\
All of $W\offb$ & 7\% & $-5.34$ & $-24.98$ \\
\bottomrule
\end{tabular}
\end{table}

\subsection{Absolute Logit Scale}
\label{app:zloss}
\begin{figure}[h]
\centering
\includegraphics[width=\textwidth]{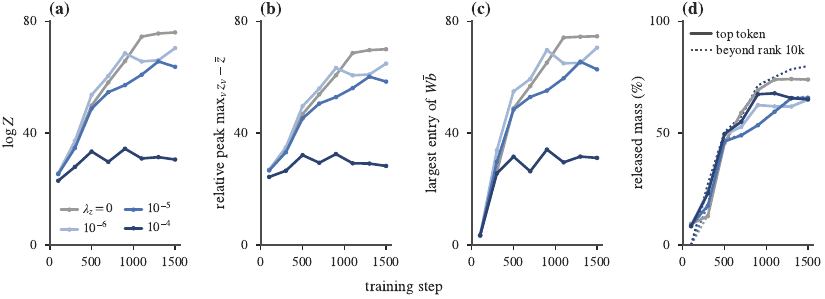}
\caption{A $z$-loss penalty on $\log Z$ compresses absolute logit scale, relative peak, and common-mode vocabulary projection without preventing well capture. Blue lines show three strengths of $\lambda_z$ on the Llama-2 flagship setting, darker for stronger, and the gray line shows the untreated run. (a) The strongest penalty, $\lambda_z=10^{-4}$, lowers $\log Z$ from 76.0 to 30.5 by step 1500. (b) It lowers the relative peak logit by the same factor of 2.5, from 70.0 to 28.2. (c) It lowers the largest entry of $W\offb$ from 74.6 to 31.1. (d) The top token still takes 65 to 66\% of the released mass at every strength, against 74\% untreated, and at $\lambda_z=10^{-4}$ the mass beyond rank 10{,}000 (dotted) rises to 80\%.}
\label{fig:zloss}
\end{figure}
One alternative hypothesis is that collapse is a symptom of the output distribution's absolute scale drifting upward, since a softmax is unchanged when the same constant is added to every logit. We test it by penalizing $\log Z = \log \sum_v \exp(z_v)$, a $z$-loss as in \citet{Chowdhery2023-hc}, at increasing strength on the Llama-2 flagship setting (Figure~\ref{fig:zloss}). The strongest coefficient compresses $\log Z$ 2.5-fold, shrinks the relative peak logit and the common-mode vocabulary projection, and lowers the margin by about one third relative to the untreated run. Held-out ranking accuracy is essentially unchanged, and well capture stays at 65 to 66\% of the released mass, against 74\% untreated. At this coefficient the $z$-loss gradient reaches 19 to 43\% of the ranking-gradient norm, yet collapse persists. We treat this as supporting rather than conclusive evidence, because a raw gradient-norm ratio is not the same as its realized contribution to the optimizer update. Together with \S\ref{sec:localizing}, the result indicates that the operative quantity is not the overall output scale but the offset's projection onto the vocabulary, which a $z$-loss constrains only indirectly and incompletely.

\subsection{Dose Along Training and the Objective's Demand}
\label{app:dose}
We summarize each snapshot by its median dose, the median gain divided by the median base gap over the rejected positions. In all nine lineages traced every 50 steps, the median dose passes one before the well takes over, and the takeover itself takes 100 to 300 steps (Figure~\ref{fig:dose}). Across the sweep over $\beta$, the final dose passes one only under per-token averaging with a demand above the base gap (Figure~\ref{fig:demand}). The standard recipes build the same kind of offset as the collapsing runs, at a smaller dose (Figure~\ref{fig:std-recipes}).

\begin{figure}[h]
\centering
\includegraphics[width=\textwidth]{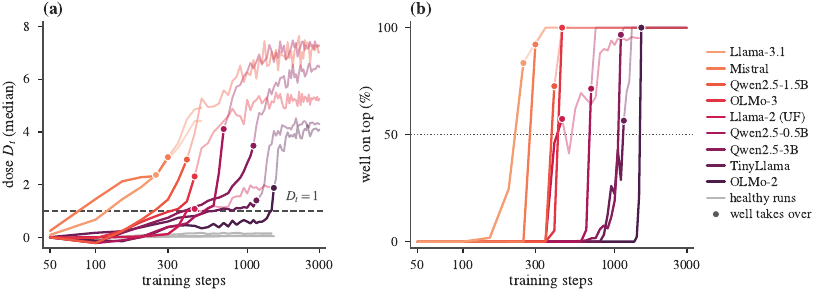}
\caption{Dose along natural training runs, with a dot at the first snapshot at which the well is the top token at a majority of rejected positions. (a) In all nine collapsing lineages, the median dose $D_t = g_t/m_t$, measured every 50 steps, passes one before the dot and lies between 1.1 and 4.1 at the dot, while the three summed-score runs that stay healthy, in gray, remain below 0.2. (b) The share of rejected positions with the well on top climbs from under 10\% to a majority within 100 steps in seven lineages, and within 200 and 300 steps in Qwen2.5-3B and TinyLlama.}
\label{fig:dose}
\end{figure}

\begin{figure}[h]
\centering
\includegraphics[width=\textwidth]{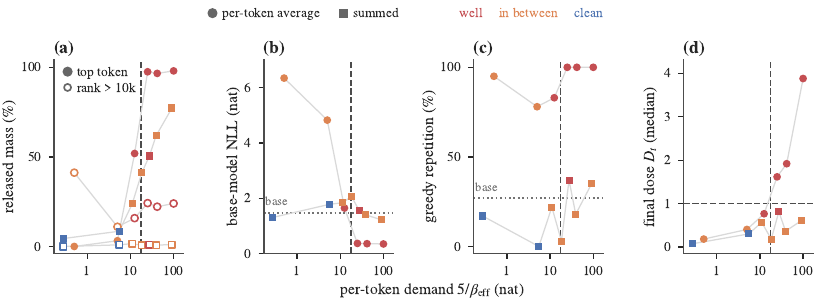}
\caption{Outcome of pairwise training on UltraFeedback against the per-token margin the objective demands, $5/\beta_{\text{eff}}$ with $\beta_{\text{eff}}=\beta$ for per-token-average scores and $\beta|y|$ for summed scores. The dashed vertical line marks the median base gap at the rejected positions, 17.5 nat, and dotted lines mark the base model at 1.48 nat and 27\%. (a) Under per-token averaging the well takes more than 95\% of the released mass once the demand passes the base gap, whereas summed scores concentrate mass gradually and send almost none beyond rank 10{,}000. (b) Samples from the two lowest-demand per-token runs have a high negative log-likelihood under the base model, samples from per-token runs beyond the base gap a low one, and summed runs stay near the base model. (c) Every per-token run repeats under greedy decoding on at least 78\% of prompts, against at most 37\% for summed runs. (d) The final dose passes one only under per-token averaging with a demand above the base gap. Red marks runs with more than half of the released mass on one token and at least 30\% degenerate samples, blue marks clean runs, and orange the runs in between. Doses are for each run's leading token, or for \texttt{Conf} when that token takes less than 5\% of the released mass.}
\label{fig:demand}
\end{figure}

\begin{figure}[h]
\centering
\includegraphics[width=\textwidth]{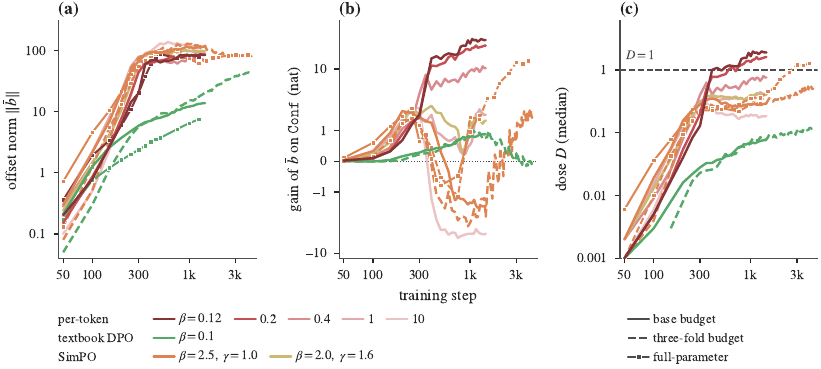}
\caption{Standard recipes build the same offset at a smaller dose. All runs train Llama-2-7B pairwise on UltraFeedback and are plotted against training step on log axes. Per-token averaging at $\beta\le0.2$ collapses, while textbook DPO at $\beta=0.1$ and SimPO at $\beta\in\{2.0,2.5\}$ do not within the base budget. (a) By the end of the base budget, SimPO has built an offset about as long as that of the collapsing runs, and textbook DPO one at most a sixth as long. (b) Only per-token averaging at $\beta\le0.2$ gives \texttt{Conf}, the well of the collapsing run, a logit gain above 20 nat, while textbook DPO and SimPO stay below 3 nat within the base budget. (c) By the end of the base budget, the doses of textbook DPO and SimPO sit four to twenty-five times below those of the collapsing runs, and full-parameter SimPO crosses $D=1$ at step 3000 under the three-fold budget. Squares mark full-parameter fine-tuning, and dashed lines mark three-fold-budget runs, which train for nine epochs under their own schedule. Full-parameter per-token averaging, which collapses onto the space token, and full-parameter textbook DPO appear only in (a).}
\label{fig:std-recipes}
\end{figure}

\subsection{Full-Parameter Onset}
\label{app:fullft}
Under full-parameter fine-tuning, onset tracks the integrated learning rate, the learning rate summed over the steps taken. Two Llama-2 runs, at learning rates $5\times10^{-7}$ and $10^{-6}$ under a 1512-step schedule, nucleated once it reached about $1.7\times10^{-4}$. Before running it, we then wrote down that a run at $10^{-6}$ under the 504-step schedule would nucleate between steps 200 and 300, and it nucleated at step 250. In all full-parameter runs with the per-token objective, the offset alone makes the well the top token at 98 to 100\% of positions, while the fine-tuned unembedding rows alone do so at no more than 1.5\%.

\subsection{A Supervised Fine-Tuning Control}
\label{app:sft}
To check whether the offset needs a ranking term at all, we trained Llama-2-7B on the driving lists at three-fold budget with only the NLL term on the chosen response. Its offset reaches a norm of 56, and the offsets on chosen and rejected responses have a cosine of 0.99. However, the largest entry of $W\offb$ is only 3.4 and the top token takes 8.6\% of the released mass, so no well forms. Plain fine-tuning thus already builds a shared offset, while the ranking term decides where it points.

\subsection{SimPO from a Supervised Fine-Tuning Checkpoint}
\label{app:sftstart}
We start from princeton-nlp/Llama-3-Base-8B-SFT, the Llama-3-8B model that the SimPO authors fine-tuned on UltraChat \citep{meng2024simpo}, and train it with their Llama-3-Base configuration of $\beta=2.0$, $\gamma=1.0$, a learning rate of $6\times10^{-7}$, a global batch of 128, a cosine schedule with a warmup ratio of 0.1, and full-parameter updates. The data are our 15.7k UltraFeedback pairs rendered in the model's chat template. The base budget of 369 steps makes three passes over these pairs, as many examples as about three quarters of one epoch of the official training data, and the three-fold budget of 1107 steps makes nine. We keep fp32 master weights with an 8-bit AdamW optimizer, where the authors use bf16 weights and AdamW.

Before training, we measured the base gap of three candidate wells at the rejected positions. Relative to Llama-3-8B, the SFT checkpoint enlarges the median base gap by 22\% for \texttt{Conf} and by 10\% for \texttt{Apache} and shrinks it by 2\% for \texttt{Sharper}. Transplanting the offset of the Llama-3.1 full-parameter SimPO run at three-fold budget, whose well is \texttt{Apache}, gives a median dose of 1.12 on the SFT checkpoint against 1.27 on Llama-3-8B. By the dose, supervised fine-tuning should therefore delay a well rather than prevent it.

Table~\ref{tab:sftstart} reports both budgets. No well forms, and held-out ranking improves over the SFT checkpoint at both budgets. From one-fold to three-fold budget, the offset grows from 32.5 to 77.3 while its largest vocabulary entry stays small. At the base budget it lifts the ordinary \texttt{Conf} family by about 2 nat, and at three times the budget a flat spread of infrequent multilingual and code fragments by about 4 nat. SimPO from the base models lifts its leading tokens by about 18 nat. No token in the vocabulary reaches a median dose of 0.4, while SimPO from the base models carries \texttt{Apache} on Llama-3.1 and \texttt{Conf} on Llama-2 past one. Meanwhile the samples move away from the SFT model, whose per-token negative log-likelihood of them rises from 1.35 to 3.86. Against the SFT model's own samples for the same 300 held-out prompts, Skywork-Reward-V2 and ArmoRM prefer the one-fold run's samples on 63\% and 65\% of prompts and the three-fold run's on 34\% and 33\%. The held-out chosen responses win on 73\% and 72\% (Table~\ref{tab:rm}).

\begin{table}[h]
\caption{SimPO from the public SFT checkpoint with the official hyperparameters. NLL is the SFT model's per-token negative log-likelihood of sampled responses, the dose is for \texttt{Conf}, and held-out pairwise accuracy is 0.619 for the SFT checkpoint itself.}
\label{tab:sftstart}
\centering
\scriptsize
\begin{tabular}{lcccccccc}
\toprule
Budget & Steps & Top token & Capture / tail & Degen. & NLL & $\|\offb\|$ & Dose & Held-out top-1 / pairwise \\
\midrule
$1\times$ & 369 & \texttt{,} & 2.5\% / 0.6\% & 1.3\% & 1.35 & 32.5 & 0.15 & 0.41 / 0.67 \\
$3\times$ & 1107 & \texttt{(} & 2.7\% / 4.1\% & 0.0\% & 3.86 & 77.3 & 0.19 & 0.58 / 0.75 \\
\bottomrule
\end{tabular}
\end{table}

\subsection{External Reward Models}
\label{app:rm}
Table~\ref{tab:rm} lists the reward-model evaluation behind \S\ref{sec:judge} and the Win column of Table~\ref{tab:objectives}. Each run generates one sample for each of the same 300 held-out UltraFeedback prompts, at temperature 0.7 with nucleus sampling at 0.95 and at most 256 new tokens. Each sample is compared with the sample of the model the run starts from for the same prompt, which is the base model except for the two runs from the SFT checkpoint. The chosen responses of the held-out lists win against the base model on 96\% and 95\% of prompts under the two reward models. Adding the SimPO offset alone to the base model, as in \S\ref{sec:causal}, keeps its samples about as likely under the base model as the base model's own samples. They cost about 0.6 nat per token, against 3.1 to 4.3 for the trained SimPO models, and when they fail they fall into ordinary repetition loops rather than strings of unrelated words.

\begin{table}[h]
\caption{External reward models on 300 held-out UltraFeedback prompts. Win is the share of prompts on which the reward model prefers the run's sample to that of the model the run starts from. Pairwise is held-out pairwise accuracy. Top-10 is the share of released mass that goes to the base model's ten most likely tokens under teacher forcing. The base Llama-2-7B has a held-out pairwise accuracy of 0.60, the base Llama-3.1-8B 0.61, and the SFT checkpoint 0.62. The last four Llama-2 rows train on four-response lists and the others on pairs.}
\label{tab:rm}
\centering
\footnotesize
\begin{tabular}{lcccc}
\toprule
Run & Skywork win & ArmoRM win & Pairwise & Top-10 \\
\midrule
\multicolumn{5}{l}{\emph{Llama-2-7B}} \\
Per-token ranking, $\beta=0.12$ & 25\% & 6\% & 0.68 & 0.02 \\
\quad with a reference model, $1\times$ / $3\times$ & 26\% / 17\% & 7\% / 3\% & 0.68 / 0.71 & 0.04 / 0.05 \\
\quad full-parameter & 24\% & 45\% & 0.70 & 0.14 \\
SimPO, $1\times$ / $3\times$ & 7\% / 6\% & 3\% / 4\% & 0.67 / 0.73 & 0.04 / 0.04 \\
SimPO, $\beta=2.0$, $\gamma=1.6$ & 8\% & 6\% & 0.67 & 0.05 \\
SimPO, full-parameter, $1\times$ / $3\times$ & 10\% / 7\% & 6\% / 4\% & 0.70 / 0.76 & 0.10 / 0.05 \\
Textbook DPO, $1\times$ / $3\times$ & 68\% / 66\% & 67\% / 65\% & 0.63 / 0.64 & 0.77 / 0.77 \\
Textbook DPO, full-parameter & 66\% & 67\% & 0.63 & 0.80 \\
NLL only, $\lambda=1$ & 64\% & 64\% & 0.63 & 0.75 \\
Centering only & 20\% & 23\% & 0.68 & 0.42 \\
Recipe, $\lambda_{\text{NLL}}=1.0$ / $0.1$ & 65\% / 52\% & 66\% / 55\% & 0.63 / 0.66 & 0.74 / 0.45 \\
Centering and loudness penalty & 17\% & 15\% & 0.69 & 0.17 \\
\midrule
\multicolumn{5}{l}{\emph{Llama-3.1-8B}} \\
Per-token ranking, $\beta=0.12$ & 11\% & 5\% & 0.74 & 0.02 \\
SimPO, full-parameter, $1\times$ / $3\times$ & 13\% / 3\% & 14\% / 5\% & 0.78 / 0.79 & 0.01 / 0.03 \\
\midrule
\multicolumn{5}{l}{\emph{Llama-3-8B from the SimPO authors' SFT checkpoint}} \\
SimPO, official recipe, $1\times$ / $3\times$ & 63\% / 34\% & 65\% / 33\% & 0.67 / 0.75 & 0.67 / 0.46 \\
\bottomrule
\end{tabular}
\end{table}

\begin{figure}[h]
\centering
\includegraphics[width=\textwidth]{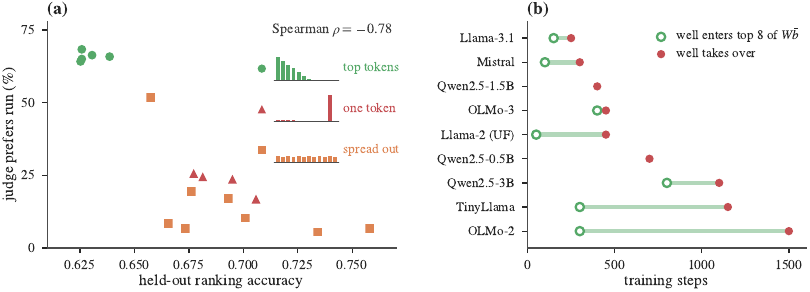}
\caption{(a) Across 17 UltraFeedback runs, the better a run ranks held-out responses, the less often a public reward model prefers its samples to the base model's; the glyphs show where each group's released mass goes. (b) In seven of nine collapsing lineages, the eventual well enters the eight largest entries of $W\offb$ before it takes over the rejected positions.}
\label{fig:margin}
\end{figure}

\begin{figure}[h]
\centering
\includegraphics[width=\textwidth]{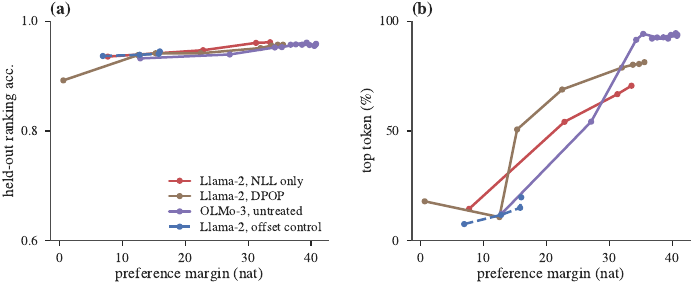}
\caption{(a) Held-out ranking accuracy and (b) the top token's share of released mass against the preference margin along training. Ranking saturates early for every run, while the top token's share keeps growing with the margin only in the runs without offset control.}
\label{fig:margin-app}
\end{figure}

\subsection{Shapes of Well That Standard Rulers Miss}
\label{app:rulers}
We compared the standard $n$-gram repetition ruler with our mass-flow and word-stem measurements on every run. The standard ruler correctly flags the flagship run's repetition loop, in which 96\% of samples degenerate, but it misses three other shapes of well. A \emph{latent well} appears only after the visible answer ends, meaning greedy decoding within a length-matched budget never reaches it. A \emph{non-repeating well} keeps changing its surface form, for example a code-path continuation, while still capturing most of the released mass. The ruler then reports little repetition despite severe collapse. A \emph{word-family ring} lacks a single dominant token, because a cluster of low-frequency, morphologically related tokens takes the released mass together. A ruler that counts single tokens therefore underreports the failure. Furthermore, these standard rulers miss the diffuse degradation from \S\ref{sec:demand}, which leaves no repetitive patterns to count but is exposed by evaluating the base model's likelihood of the generated samples. Any conventional ruler used in isolation can therefore pass a model that the offset and mass-flow measurements from \S\ref{sec:measurement} reveal has already collapsed.

\subsection{Supplementary Measurements}
\label{app:supp}
These measurements support claims of the main text, grouped by section.

\textbf{The offset is shared (\S\ref{sec:localizing}).} Recomputed on held-out text with phrasing the model never saw, the flagship offset keeps a cosine of 0.85 with the original.

\textbf{Dose threshold (\S\ref{sec:dose}).} When the well takes over along natural training runs, it goes from under 10\% of the rejected positions to a majority within 100 to 300 steps (Figure~\ref{fig:dose}b).

\textbf{Anomalous rows (\S\ref{sec:geometry}).} For every single-token well, we looked up the norm of its unembedding row and counted how often the token occurs in our training text and in WikiText-103. Most single-token wells of the adapter runs on the driving data never appear in our training text. In 13 bases, the highest-norm row belongs to a rare token, and its gain $g_t$ can overtake base gaps of 17 to 39 nat. Under full-parameter fine-tuning on UltraFeedback, the Llama family collapses onto the bare space token, and there too the offset alone makes the well the top token at nearly every position.

\textbf{Losses and demand (\S\ref{sec:demand}).} On the driving lists, the per-token loss with a reference model collapses onto \texttt{textt} at the base budget and onto a rare Cyrillic token at the three-fold budget. Per-token training at $\beta$ of 1 and 10, whose demand is also low, degrades in the same diffuse way (Figure~\ref{fig:demand}). With three times the budget, the SimPO offset concentrates its projection on a few vocabulary rows even where it becomes shorter. By the end of full-parameter training, its leading token takes 27 to 45\% of the released mass (Table~\ref{tab:objectives}).

\textbf{Ranking cost of the recipe (\S\ref{sec:recipe-validation}).} Eight base models have treated runs in our held-out evaluation. We compare each treated run with the untreated run of the same base and budget, or with the NLL-only run where no untreated run was scored. Every treated run but one lies between 1.8 points below and 0.1 points above its comparison run. The exception is the loudness variant of the recipe on Gemma-2-9B at the base budget, which barely lowers the rejected responses in training and scores 3.1 points below the untreated run, while the combined penalty on the same base scores 0.3 points below it.

\textbf{Margin along training (\S\ref{sec:dissociation}).} The margin reaches three or four times its value at the point where ranking saturates. At a matched margin of 12 to 16 nat, Llama-2 runs without offset control have already channeled 42 to 51\% of the released mass into a single well, against 8 to 20\% for offset-controlled runs (Figure~\ref{fig:margin-app}). The margin at which capture takes off also differs by base model, from 13 to 27 nat on six-fold OLMo-3 to 44 to 59 nat on OLMo-2, so the margin cannot serve as a stopping criterion.

\subsection{Seeds and Scope}
\label{app:scope}
Both untreated seeds of the flagship settings collapse onto the same well as the first, \texttt{textt} on the driving lists and \texttt{Conf} on UltraFeedback. Their offset norms lie within 3\% of the first seed's, and their dose crosses one no later than the well takes over. Both treated seeds stay free of a well, with 10 to 12\% residual degeneration and unchanged held-out ranking. Adapter rank from 8 to 128 and attention-only adapters change when a well forms rather than whether. Across these ranks, offset norms lie between 80 and 117 on Llama-2 and Llama-3.1, and rank 8 on Llama-2 nucleates the same well at three-fold budget. The demand analysis of \S\ref{sec:demand} covers pairwise UltraFeedback training under per-token averaging with $\beta$ from 0.05 to 10, under textbook DPO with $\beta$ from 0.1 to 0.0003, and under SimPO with $\beta$ and $\gamma$ from its published tuning range. These runs use our adapter and full-parameter protocols for three or nine epochs.


\subsection{Common-Shift Invariance and the Trained Offset}
\label{app:invariance}
Write $s_\theta(x,y)$ for the score of a response, the per-token average of \S\ref{sec:setup} or the summed log-ratio of textbook DPO. Both $\mathcal{L}_{\text{PL}}$ and the DPO loss depend on the scores of a prompt's responses only through their differences, since the DPO reference term does not depend on $\theta$.

\begin{proposition}[Loss-level invariance]
\label{prop:loss}
Let $\ell$ be invariant under adding a constant to the scores of all responses of a prompt, and write a score perturbation as $\varepsilon=\bar\varepsilon+\tilde\varepsilon$ with $\bar\varepsilon(x)$ its mean over the prompt's responses. Then $\ell(s+\varepsilon)=\ell(s+\tilde\varepsilon)$, so the common part $\bar\varepsilon$ cannot be identified from the objective with any amount of preference data, and $\sum_y\partial\ell/\partial s(x,y)=0$ for every prompt.
\end{proposition}
\begin{proof}
Score differences within a prompt cancel $\bar\varepsilon$. Differentiating $\ell(s+c\mathbf 1)=\ell(s)$ in $c$ at $c=0$ gives the gradient identity.
\end{proof}

\begin{proposition}[First-order constraint on a shared hidden shift]
\label{prop:offset}
Let $W$ be the unembedding with rows $w_v$, $h_t$ the final hidden state at position $t$ and $p_t=\mathrm{softmax}(Wh_t)$. Shifting every hidden state of every response by the same $\delta$ changes a per-token-average score by
\begin{align*}
s(x,y;\delta)-s(x,y;0)&=\frac{1}{|y|}\sum_t\Big[\langle w_{y_t},\delta\rangle-\log\mathbb{E}_{v\sim p_t}e^{\langle w_v,\delta\rangle}\Big]\\
&=\langle g_y,\delta\rangle-\frac{1}{2|y|}\sum_t\mathrm{Var}_{v\sim p_t}\langle w_v,\delta\rangle+O(\|\delta\|^3),
\end{align*}
with $g_y=\frac{1}{|y|}\sum_t\big(w_{y_t}-\mathbb{E}_{v\sim p_t}w_v\big)$ the mean gradient of the score with respect to the hidden states. Hence $\nabla_\delta L|_{\delta=0}=\mathbb{E}_x\sum_y\frac{\partial\ell}{\partial s(x,y)}\,g_y$, which for a pairwise logistic loss equals $-\mathbb{E}_x[a(x)(g_{y_w}-g_{y_l})]$ with $a(x)=\beta\sigma(-\beta(\Delta_\theta-\Delta_{\mathrm{ref}}))>0$ and $\Delta_{\mathrm{ref}}=0$ without a reference model. To first order the objective constrains $\delta$ along this single direction, and on the orthogonal complement only the second-order term remains, which enters through the difference between the responses' mean per-position variances.
\end{proposition}
\begin{proof}
The first equality is the change of $\log\mathrm{softmax}$ under the logit shift $W\delta$, the second the cumulant expansion of $\log\mathbb{E}\,e^{\langle w_v,\delta\rangle}$, and the gradient follows from the chain rule with coefficients that sum to zero within each prompt by Proposition~\ref{prop:loss}.
\end{proof}

Three remarks bound the statement. The constraint is first-order and local, since the softmax is nonlinear and minibatch gradients differ from the expectation, so the complement is weakly rather than exactly unconstrained. The shift $\delta$ is not a parameter but the image of parameter updates through Jacobians shared across positions and both responses. This sharing is how a component in the weakly constrained part arises at all, and gradient entanglement \citep{yuan2025common} describes the coupling that produces it. Any term that depends on absolute scores, such as an NLL term on the chosen response, breaks the invariance, which is one reason such terms delay displacement (\S\ref{sec:related}).

Table~\ref{tab:nullspace} measures how much of the trained offset lies in the weakly constrained part. For each run we add the offset $\delta$ of \S\ref{sec:carrier-def} to the base model's final hidden states on 200 held-out lists. For each list we record the change of the chosen per-token score, $a^{+}$, and of the mean rejected per-token score, $a^{-}$, and split it into a common part $c=(a^{+}+a^{-})/2$ and a differential part $d=(a^{+}-a^{-})/2$. The table reports the means of $a^{+}$, $a^{-}$, and $c$ over the lists and the mean of $|d|$, which exceeds half the difference of the mean scores because $d$ changes sign from list to list. We report the same two means for three norm-matched random directions, and the cosine between $\delta$ and $\bar u=\frac{1}{N}\sum_x\big(g_{y^{(1)}}-\frac{1}{K-1}\sum_{j=2}^{K}g_{y^{(j)}}\big)$, the base model's chosen-minus-rejected score gradient with every list and every rejected response weighted equally. For the three per-token UltraFeedback runs, which pair the chosen response with a randomly drawn rejected one, $\bar u$ is the direction of Proposition~\ref{prop:offset} with equal weights, and at the base model those weights vary across lists by only 5 to 7\% of their mean. For the driving run $\bar u$ approximates that direction, since the $K=4$ loss weights the rejected responses by rank. In the four collapsing runs the common part is 20 to 530 times the differential part and the cosine is at most 0.16. The offset carries the collapse-producing component without accounting for all of the likelihood displacement. On the driving run it alone lowers the chosen score by 65 nat while training lowered it by 21 on the same lists on the same lists, so the position-specific residual partly compensates its effect on the chosen response, consistent with \S\ref{sec:localizing}. For the two standard-recipe runs both parts are small, below 0.02 nat under textbook DPO and 0.25 and 0.06 nat under SimPO, because their quiet offsets lift no token near the top.

\begin{table}[h]
\caption{The trained offset $\delta$ added to the base model's final hidden states on 200 held-out lists. $a^{+}$ and $a^{-}$ are the mean changes of the chosen and mean rejected per-token scores in nat, $c$ the mean of $(a^{+}+a^{-})/2$ and $|d|$ the mean of $|a^{+}-a^{-}|/2$ over lists, ``random'' the same two means for three norm-matched random directions, and cos the cosine between $\delta$ and $\bar u$, the base model's chosen-minus-rejected score gradient.}
\label{tab:nullspace}
\centering
\scriptsize
\resizebox{\textwidth}{!}{%
\begin{tabular}{llccccccc}
\toprule
Base & Data, objective & $\|\delta\|$ & $a^{+}$ & $a^{-}$ & $c$ & $|d|$ & random $|c|$ / $|d|$ & cos \\
\midrule
Llama-2-7B & UltraFeedback, per-token $\beta=0.12$ (collapses) & 84 & $-14.5$ & $-15.2$ & $-14.9$ & 0.68 & 0.30 / 0.06 & 0.03 \\
Llama-2-7B & driving, $K=4$ (collapses) & 93 & $-65.1$ & $-65.2$ & $-65.2$ & 0.12 & 0.49 / 0.02 & 0.16 \\
Qwen2.5-7B & UltraFeedback, per-token $\beta=0.12$ (collapses) & 282 & $-26.0$ & $-27.1$ & $-26.6$ & 1.36 & 2.15 / 0.31 & $-0.02$ \\
Llama-3.1-8B & UltraFeedback, per-token $\beta=0.12$ (collapses) & 109 & $-29.2$ & $-30.1$ & $-29.6$ & 0.94 & 0.54 / 0.10 & 0.05 \\
Llama-2-7B & UltraFeedback, textbook DPO $\beta=0.1$ & 14 & $-0.005$ & $-0.013$ & $-0.009$ & 0.014 & 0.011 / 0.008 & 0.07 \\
Llama-2-7B & UltraFeedback, SimPO $\beta=2.5$, $\gamma=1.0$ & 114 & $-0.23$ & $-0.26$ & $-0.25$ & 0.06 & 0.57 / 0.09 & 0.01 \\
\bottomrule
\end{tabular}}
\end{table}

\end{document}